\documentclass{article} 
\usepackage{iclr2027_conference,times}

\usepackage{hyperref}
\usepackage{graphicx}
\usepackage{url}            
\usepackage{booktabs}       
\usepackage{amsfonts}       
\usepackage{nicefrac}       
\usepackage{microtype}      
\usepackage{xcolor}         
\usepackage{amsmath, amssymb, amsthm}
\usepackage{cleveref}
\usepackage{subcaption}
\usepackage{algorithm}
\usepackage{algpseudocode}
\usepackage{wrapfig}
\usepackage{empheq}
\usepackage{enumitem}
\usepackage{float}

\usepackage{tabularx}
\usepackage{xurl}

\usepackage{listings}
\definecolor{codeblue}{HTML}{0000CC}
\definecolor{codecomment}{HTML}{6FA3A0}
\definecolor{codered}{HTML}{CC3333}
\definecolor{codefunc}{HTML}{D65A9E}
\lstdefinestyle{pythonstyle}{
    language=Python,
    basicstyle=\ttfamily\footnotesize,
    keywordstyle=\color{codeblue}\bfseries,
    commentstyle=\color{codecomment},
    stringstyle=\color{codered},
    morekeywords=[2]{fixed_pairing,dynamic_pairing,randint},
    keywordstyle=[2]\color{codefunc},
    showstringspaces=false,
    columns=fullflexible,
    keepspaces=true,
    breaklines=true
}

\title{Zero-Flow Two-Sample Tests}

\author{Yakun Wang$^{1,*}$, Leyang Wang$^{2,3,\dagger}$, Song Liu$^{1,\ddagger}$, Taiji Suzuki$^{2,3,\S}$}

\theoremstyle{plain}
\newtheorem{theorem}{Theorem}[section]
\newtheorem{proposition}{Proposition}[section]
\newtheorem{lemma}{Lemma}[section]
\newtheorem{corollary}{Corollary}[section]

\theoremstyle{definition}
\newtheorem{definition}{Definition}[section]

\theoremstyle{remark}
\newtheorem{remark}{Remark}

\crefname{appendix}{Appendix}{Appendices}
\Crefname{appendix}{Appendix}{Appendices}

\newcommand{\R}{\mathbb{R}}
\newcommand{\E}{\mathbb{E}}

\newcommand{\var}{\operatorname{Var}}

\newcommand{\ZFD}{\textsc{ZFD}}

\newcommand{\rmtr}{\mathrm{tr}}
\newcommand{\rmte}{\mathrm{te}}

\newcommand{\Law}{\operatorname{Law}}

\newcommand{\bzero}{\mathbf{0}}

\newcommand{\bh}{\mathbf{h}}
\newcommand{\bg}{\mathbf{g}}
\newcommand{\bj}{\mathbf{j}}

\newcommand{\bv}{\mathbf{v}}
\newcommand{\bu}{\mathbf{u}}

\newcommand{\bx}{\mathbf{x}}
\newcommand{\by}{\mathbf{y}}
\newcommand{\bz}{\mathbf{z}}

\newcommand{\bJ}{\mathbf{J}}

\newcommand{\bR}{\mathbf{R}}

\newcommand{\calE}{\mathcal{E}}
\newcommand{\calF}{\mathcal{F}}
\newcommand{\calP}{\mathcal{P}}
\newcommand{\calR}{\mathcal{R}}
\newcommand{\calS}{\mathcal{S}}

\newcommand{\fraks}{\mathfrak{s}}

\newcommand{\dd}{\mathrm{d}} 

\iclrfinalcopy 
\begin{document}

\maketitle

\begingroup  
\vspace{-2em} 
$^1$University of Bristol \quad
$^2$University of Tokyo\\
$^3$Center for Advanced Intelligence Project, RIKEN\\
$^*$\texttt{\href{mailto:yakun.wang@bristol.ac.uk}{yakun.wang@bristol.ac.uk}} \\
$^\dagger$\texttt{\href{mailto:leyang@g.ecc.u-tokyo.ac.jp }{leyang@g.ecc.u-tokyo.ac.jp }}\\
$^\ddagger$\texttt{\href{mailto:song.liu@bristol.ac.uk}{song.liu@bristol.ac.uk}} \\
$^\S$\texttt{\href{mailto:taiji@mist.u-tokyo.ac.jp}{taiji@mist.u-tokyo.ac.jp}}\\
\vspace{1em}
\par
\endgroup

\begin{abstract}
Motivated by the success of modern flow-based generative models in modeling complex data, we study two-sample testing through the lens of flow-based methods. We propose the Zero-Flow Two-Sample Test (ZF2ST), built on the zero-flow criterion, which characterizes distributional equality through a time-reversal antisymmetry of a learnable velocity field.
We extend this criterion and further develop the Zero-Flow Discrepancy, an identifying discrepancy that controls the Wasserstein distance, and derive a variational representation in terms of a witness function.
This representation naturally leads to a witness-based test whose power is governed by the signal-to-noise ratio (SNR), allowing direct power maximization for witness learning.
ZF2ST learns the witness on one data split and performs testing on held-out samples, thereby maintaining Type-I error control and admitting a simple asymptotic null distribution. 
Experimentally, ZF2ST performs competitively across various synthetic and real-world benchmarks, while showing particularly strong performance in distinguishing image distributions from different sources.
\end{abstract}

\section{Introduction}
Two-sample testing is a fundamental statistical problem for determining whether two sets of observations arise from the same distribution \citep{smirnov1939estimation,friedman1979multivariate,fasano1987multidimensional}.
A broad line of work has extended it beyond simple, low-dimensional settings to high-dimensional and structured data \citep{borgwardt2006integrating,sugiyama2011least,szekely2013energy,chwialkowski2015fast,chen2017new,ghoshdastidar2018practical}. These advances have made two-sample testing a versatile tool in modern machine learning, with applications including dataset comparison, generative model evaluation, distribution shift detection, and privacy auditing \citep{gretton2012kernel,lopez2016revisiting,sutherland2016generative,rabanser2019failing,ribero2026regularized}.

A dominant framework for nonparametric two-sample testing is the maximum mean discrepancy (MMD; \citealp{gretton2012kernel}), which compares kernel mean embeddings of the two distributions \citep{muandet2017kernel}.
However, its power depends critically on the kernel, with generic choices often degrading rapidly in high dimensions \citep{ramdas2015decreasing}.
Early work addressed this by optimizing kernels and test features for power \citep{gretton2012optimal,jitkrittum2016interpretable,sutherland2016generative}. A major advance along this direction was the deep-kernel test of \citet{liu2020learning}, which parameterizes the kernel with a neural network and trains it to directly maximize testing power, enabling kernel tests to detect complex, high-dimensional alternatives such as differences between natural image distributions. Subsequent work has learned witness functions explicitly \citep{kubler2022witness,kubler2022automl} and combined multiple kernels to avoid committing to a single choice \citep{schrab2023mmd,biggs2023mmd,zhou2026dual}.
These developments highlight that a central challenge in modern two-sample testing is to learn the structure of distributional differences, which may be subtle, embedded in high-dimensional data, and difficult to characterize a priori.

Recent advances in generative modeling offer a promising route to this challenge. Flow-based generative models have achieved remarkable success in capturing data distributions that are multimodal, highly structured, and rich in fine-grained detail \citep{lipmanflow,liu2022flow,albergo2025stochastic}.
Beyond sample generation, they have also proven effective for statistical tasks such as missing-data imputation, mutual-information estimation, and simulation-based inference \citep{yu2025missing,butakov2025fmmi,wildberger2023flow}.
Motivated by these advances, we exploit the modeling capabilities of flow models to discern structured differences between distributions for two-sample testing.

Our approach builds on the zero-flow criterion \citep{wang2026zero}, which characterizes distributional equality through the flow-matching velocity field. The original criterion requires non-vanishing characteristic functions \ref{asmp:nonzero_cf} and was developed primarily for representation learning, rather than as a statistical discrepancy or two-sample test. Recently, \citet{shlenskii2026midpoint} study a corresponding divergence for generative modeling under bounded-support assumptions.
In this work, we extend the zero-flow criterion to all $p,q\in\mathcal P_1(\R^d)$
\footnote{$\calP_{1}(\R^d):=\left\{ p \in \calP(\R^d) \mid \int_{\R^d} \|x\| \, \dd p(x) < \infty \right\}$, $\calP(\R^d)$ is the set of all probability measures on $\R^d$.}
and introduce the Zero-Flow Discrepancy (ZFD), whose validity is established through Wasserstein control. We then derive a variational form of ZFD, yielding a simple two-sample test with asymptotic normality and an explicit power criterion.
We call the resulting procedure the Zero-Flow Two-Sample Test (ZF2ST). It supports either standard flow-matching regression or direct optimization of testing power.
We further introduce variance reduction to improve test power and a wild-bootstrap calibration with finite-sample validity. 
Finally, ZF2ST performs strongly across a range of hypothesis testing problems, with particular strength on high-dimensional image data comparisons. We summarize our contributions as follows.
\begin{itemize}
\item We introduce the \textit{Zero-Flow Discrepancy} (ZFD), an identifying discrepancy with Wasserstein control that can be readily estimated through standard flow-matching regression.

\item We propose ZF2ST, a ZFD-based test with asymptotic normality and an explicit power criterion, offering two training objectives based on either flow-matching regression or direct power maximization.

\item We demonstrate that ZF2ST effectively detects subtle and complex distributional alternatives across synthetic and real-world hypothesis testing problems, with particular strength on challenging image data comparisons.
\end{itemize}

\section{Background}
\label{sec:setup}

\textbf{Two-sample testing.}
The goal of two-sample testing is to determine whether the two samples are drawn from the same distribution.
Let \(p\) and \(q\) be Borel probability measures on \(\R^d\). Given two sets of samples \(\calS_p=\{\bx_i\}_{i=1}^n\) and \(\calS_q=\{\by_j\}_{j=1}^m\) drawn i.i.d. from $p$ and $q$ respectively, we test the null hypothesis \(H_0:p=q\) against the alternative \(H_1:p\neq q\). Throughout the paper we will assume $n=m$. A test is established on a test statistic $\zeta = \zeta(\calS_p,\calS_q)$ typically based on statistical discrepancies, where larger values indicate stronger evidence against \(H_0\).
Let \(\zeta_0\) denote the test statistic under $H_0$, whose distribution can be calibrated by bootstrap, permutation, or asymptotic theory \citep{efron1985bootstrap,good2013permutation,gretton2012kernel}.
The right-tail \(p\)-value is \(p_{\mathrm{val}}=\Pr_{H_0}(\zeta_0\geq \zeta)\) and \(H_0\) is rejected at significance level \(\alpha \in (0,1)\) whenever \(p_{\mathrm{val}} \leq \alpha\). A desirable test controls the resulting Type-I error at level $\alpha$ while achieving high power, i.e., a high probability of rejecting $H_0$ under $H_1$.

\textbf{Flow matching.}
Flow Matching (FM; \citealp{lipmanflow, liu2022flow}) learns a vector field that transports samples from a prior distribution $p$ to a target distribution $q$. 
Let $X \sim p$ and $Y \sim q$ be independently drawn.
We define the linear interpolation
\begin{equation} \label{eq:linear_interpolation}
    Z_t = (1-t)X + tY, \; t\in[0,1],
\end{equation}
with $Z_0=X$ and $Z_1=Y$. This interpolation induces a path of marginal distributions $p_t:= \Law(Z_t)$ connecting $p=p_0$ and $q=p_1$. When the second moments are finite, FM learns the associated time-dependent velocity field by minimizing the regression loss $\calR(\bu):=\E_{t \sim \operatorname{Unif}[0,1], X, Y} \| Y - X - \bu(Z_t, t) \|^2,$ whose optimal solution is the conditional expectation
\begin{equation} \label{eq:fm_loss}
    \bv_t(\bz) = \E[Y - X | Z_t = \bz].
\end{equation}
Under standard regularity conditions \citep{lipmanflow}, the same marginals $p_t$ are realized by the ODE $\mathrm{d}Z'_t = \bv_t(Z'_t)\,\mathrm{d}t$, initialized from $Z'_0 \sim p$.
Numerically integrating this ODE therefore transports samples from $p$ to $q$. Thus, FM converts generative modeling into the problem of learning the velocity field.

\textbf{Zero-flow criterion.}
Recently, a key symmetry of the flow-matching velocity field was studied by \citet{wang2026zero} and \citet{shlenskii2026midpoint}.
\begin{proposition}[Zero-flow criterion; \citet{wang2026zero}]
Suppose $p,q \in \calP_1(\R^d)$ and satisfy
\label{prop:zf_criterion}
    \begin{enumerate}[label = \textbf{(A\arabic*)}, start=0]
    \item \textbf{Non-vanishing characteristic functions:} $\varphi_p(\omega) \neq 0, \varphi_q(\omega) \neq 0$. \label{asmp:nonzero_cf}
    \footnote{$\varphi_p(\omega) := \E_{X \sim p}[e^{i\omega^\top X}]$.}
\end{enumerate}
Let $X \sim p$ and $Y \sim q$ be drawn independently. Then $\bv_{1/2}(\bz)=\bzero, \forall \bz$ if and only if $p=q$.
\end{proposition}
Since $\bv_{1/2}(\bz)=\E[Y-X | (X+Y)/2=\bz]$, the exchangeability of $X$ and $Y$ under $p=q$ immediately implies $\bv_{1/2}(\bz)=\bzero$, while the converse holds under \ref{asmp:nonzero_cf}. 
\citet{wang2026zero} further developed the criterion for conditional independence and representation learning, whereas \citet{shlenskii2026midpoint} used it for generative modeling.

\section{From Zero-Flow to an Identifying Discrepancy}
A basic requirement for a general two-sample discrepancy is that it separates $p=q$ from distinct alternatives, as achieved by MMD with characteristic kernels \citep{sriperumbudur2011universality}.
The zero-flow criterion provides a natural starting point: under the null, the midpoint velocity satisfies $\bv_{1/2}(\bz)=\bzero$. However, its converse identification relies on \ref{asmp:nonzero_cf}.
We therefore seek a quantity that retains the \textit{full-time} structure of zero flow while identifying $p=q$ over a broader distribution class.

\subsection{Extending the Zero-Flow Criterion}
Indeed, we can extend zero-flow using the notion of probability flux, which combines velocity with probability mass.
\citet{wang2026zero} point to the fact that when $p=q$ the population velocity exhibits the time-reversal antisymmetry $\bv_t(\bz)=\E[Y-X|Z_t = \bz] = \E[X - Y|Z_{1-t} = \bz] = -\bv_{1-t}(\bz)$.
This follows from the exchangeability of $X$ and $Y$: swapping the two endpoints reverses the displacement while mapping the interpolation state at time $t$ to that at time $1-t$,
\begin{equation} \label{eq:symm_interpolation}
    Z_t=(1-t)X+tY
    \quad\longleftrightarrow\quad
    Z_{1-t}=(1-t)Y+tX .
\end{equation}
The same exchangeability further implies $Z_t\stackrel{d}{=}Z_{1-t}$ for all $t\in[0,1]$, and hence the marginal symmetry $p_t=p_{1-t}$.
The continuity equation \citep{ambrosio2005gradient},
\begin{equation} \label{eq:conti_eq}
    \partial_t p_t + \nabla\cdot\bigl(p_t\bv_t\bigr) = 0.
\end{equation}
couples the marginal and velocity through the \textit{probability flux} $\bj_t:=p_t\bv_t$.
\footnote{For ease of exposition, we use density notation; the corresponding measure-theoretic formulation, which does not require the existence of densities, is provided in Appendix \ref{app:prf_zflux}.}
Under $p=q$, the marginal symmetry and velocity antisymmetry therefore combine into the flux cancellation $\bj_t+\bj_{1-t}=0$.
\begin{theorem}[Extended Zero-Flow criterion] \label{thm:zflux}
    Suppose $p,q\in\calP_1(\R^d)$, and let $X\sim p$ and $Y\sim q$ be drawn independently.
    Then $\bj_t+\bj_{1-t}=\bzero, \forall t\in [0,1]$ if and only if $p=q$. 
\end{theorem}
A proof can be found in Appendix \ref{app:prf_zflux}. Unlike \Cref{prop:zf_criterion}, the extended zero-flow criterion does not require \ref{asmp:nonzero_cf}.
To see the reverse characterization, applying \cref{eq:conti_eq} to $p_t$ and $p_{1-t}$ gives
\begin{equation} \label{eq:flux_asymm}
    \partial_t (p_t-p_{1-t})
    =
    -\nabla\cdot(\bj_t+\bj_{1-t}). 
\end{equation}
Thus, if $\bj_t+\bj_{1-t}=0$ for all $t\in[0,1]$, then $p_t-p_{1-t}$ is constant in $t$; evaluating at $t=1/2$ gives $p_t=p_{1-t}$, and hence $p=q$ at the endpoints.
\begin{remark}[Relationship with Zero-Flow criterion]
At $t=1/2$, $\bj_{1/2}=p_{1/2}\bv_{1/2}=\bzero$ implies $\bv_{1/2}=\bzero$ $p_{1/2}$-a.s. Thus, \Cref{thm:zflux} recovers the condition in \Cref{prop:zf_criterion}, while the converse implication from $\bv_{1/2}=\bzero$ to $p=q$ still requires \ref{asmp:nonzero_cf}.
\end{remark}

\subsection{Zero-Flow Discrepancy}
In order to turn \Cref{thm:zflux} into a quantitative statistical discrepancy, we compare the flux at $t$ and $1-t$. 
Let $\widetilde p_t=(p_t+p_{1-t})/2$ be the symmetrized marginal and define the normalized time-varying flux residual $\bg(\cdot,\cdot): \R^d \times [0,1/2] \rightarrow \R^d$ as
\footnote{The same density convention is used here; see Appendix~\ref{app:prf_zfd_w1} for the measure-theoretic definition.}
\begin{equation} \label{eq:gt_fluxform}
    \bg(\bz,t) := \frac{ \bj_t(\bz)+\bj_{1-t}(\bz) }{ p_t(\bz)+p_{1-t}(\bz) }. 
\end{equation}
At each interpolation state $\bz$, $\bg(\bz,t)$ measures the violation of flux cancellation on $t$. We thus aggregate its squared magnitude over $\widetilde p_t$ to define the \textit{Zero-Flow discrepancy}.
\begin{definition}[Zero-Flow Discrepancy] \label{def:zfd} 
The Zero-Flow Discrepancy $\ZFD: \calP_1(\R^d) \times \calP_1(\R^d) \to 
\R \cup \{+\infty\}$ is defined as 
\begin{equation} \label{eq:zfd}
    \operatorname{ZFD}(p,q) := \E_{t\sim\operatorname{Unif}[0,\frac{1}{2}]} \E_{\bz \sim\widetilde p_t} \bigl[\|\bg(\bz,t)\|^2\bigr].
\end{equation} 
\end{definition}
Integrating \cref{eq:flux_asymm} over $t\in[0,1/2]$ shows that the accumulated flux residual recovers the endpoint distribution difference:
\begin{equation} \label{eq:pq_mismatch}
    p-q =
    \nabla\cdot
    \int_0^{1/2}
    \left(\bj_t+\bj_{1-t}\right)\,\mathrm dt.
\end{equation}
This identity leads to the following Wasserstein control.
\footnote{$
W_1(p,q)
:=
\inf_{\pi\in\Pi(p,q)}
\int_{\R^d\times\R^d}
\|\bx-\by\|\,\mathrm d\pi(\bx,\by)
$
, where $\Pi(p,q)$ is the set of couplings of $p$ and $q$.}
\begin{proposition}[Wasserstein control] \label{prop:zfd_w1} 
For any $p,q\in\calP_1(\R^d)$, $W_1^2(p,q) \le \ZFD(p,q)$. 
\end{proposition}
A proof can be found in Appendix \ref{app:prf_zfd_w1}.
Integrating \cref{eq:pq_mismatch} against any $1$-Lipschitz function and applying Cauchy-Schwarz yields a uniform bound, whose supremum is identified with $W_1$ by the Kantorovich-Rubinstein duality \citep[Ch.~5]{villani2009optimal}.
Together with the symmetry and nonnegativity of the construction, \Cref{prop:zfd_w1} immediately gives the definiteness of ZFD.

\begin{corollary}[Symmetry and definiteness of ZFD]
\label{coro:def_zfd}
For any \(p,q\in\calP_1(\mathbb R^d)\),
\[
    \ZFD(p,q)=\ZFD(q,p)\geq 0,
    \qquad
    \ZFD(p,q)=0 \iff p=q.
\]
\end{corollary}
A proof can be found in Appendix \ref{app:prf_def_zfd}. Furthermore, the normalized flux residual $\bg$ admits the following conditional expectation form.
\begin{lemma} \label{lem:gt_regform}
    Define $\widetilde Z_t=(1-B)Z_t + BZ_{1-t}$, where $B\sim\operatorname{Ber}(1/2)$.
    Then $\Law(\widetilde Z_t)=\widetilde p_t$ and
    \begin{equation} \label{eq:gt_regform}
        \bg(\bz, t) = \E[Y-X\mid \widetilde Z_t=\bz],
        \quad
        \bg(\cdot,\cdot) =
        \underset{\bu}{\arg \min} \;
        \E_{t\sim \operatorname{Unif}[0,1/2],X,Y,B}\|Y-X - \bu(\widetilde Z_t,t)\|^2.
    \end{equation}
\end{lemma}
A proof can be found in Appendix \ref{app:prf_gt_regform}. \Cref{lem:gt_regform} shows that $\bg$ can be learned using the standard flow-matching regression objective. Consequently, ZFD can be estimated by first learning $\bg$ and then squaring the resulting estimator.
We note that our residual field $\bg$ coincides with the flipped conditional field introduced by \citet{shlenskii2026midpoint}. Accordingly, up to time normalization, \Cref{def:zfd} recovers their time-integrated Midpoint Divergence. Their work establishes identification under bounded-support assumptions and uses the resulting divergence as a critic for training generative models. In contrast, our focus is on its role in two-sample testing, where directly squaring a learned residual raises a distinct statistical challenge that motivates the construction developed next.

\begin{remark}[Finite state spaces]
\Cref{def:zfd} does not directly apply to abstract finite or categorical state spaces $\mathcal{X}$, where the linear interpolation \eqref{eq:linear_interpolation} and displacement $Y-X$ are not intrinsically defined.
It therefore requires an embedding $\phi:\mathcal X\to\R^D$, such as one-hot encoding, and the resulting discrepancy depends on the induced Euclidean representation rather than
the native discrete state space.
\end{remark}

\section{From ZFD to ZF2ST}
\label{sec:zftest}
Although ZFD admits a na\"ive plug-in estimator,
under $H_0:p=q$, the population field satisfies $\bg(\bz,t)\equiv\bzero$, while its finite-sample estimate $\widehat{\bg}(\bz,t)$ generally does not.
Squaring the estimated field turns the null statistic into a quadratic functional of the estimation error. Such plug-in quadratic statistics generally do not admit a universal null distribution: their asymptotic behavior depends on the empirical-process limit of the underlying estimator and its estimation regime \citep{hall1984central,anderson1994two}. Consequently, even under $H_0$, calibrating the squared plug-in ZFD may require an estimator-specific asymptotic analysis.
We circumvent this difficulty through two decoupling steps. First, we avoid squaring the estimated field by reformulating ZFD in terms of linear directional alignment. Second, we separate witness learning from testing through data splitting.

\subsection{A Scale-Invariant Form of ZFD} \label{subsec: scale_inv_zfd}
Inspired by MMD, which admits both an RKHS norm representation in terms of the difference between kernel mean embeddings and an equivalent variational representation as the maximal alignment with a \textit{witness function} over the RKHS unit ball \citep{gretton2012kernel}, we derive an analogous scale-invariant variational representation of ZFD.
\begin{proposition}[Scale-invariant ZFD]
\label{prop:zfd_dual}
Let $t\sim\operatorname{Unif}[0,1/2]$ and set
$\mu:=\Law(\widetilde Z_t,t)$. 
Denote by $L^2(\mu;\R^d)$ the Hilbert space of measurable time-dependent vector fields $\bh:\R^d\times[0,1/2]\to\R^d$ satisfying
$\E_\mu[\|\bh(\widetilde Z_t,t)\|^2]<\infty$, equipped with the canonical inner product $\langle\cdot,\cdot\rangle_\mu$ and norm
$\|\cdot\|_\mu$.
Suppose $\ZFD(p,q)<\infty$. Then
\begin{equation}
\label{eq:zfd_dual}
\ZFD(p,q)
=
\sup_{
    \bh\in L^2(\mu;\R^d); \;
    \|\bh\|_\mu>0
}
\frac{
    \langle\bh,\bg\rangle_\mu^2
}{
    \|\bh\|_\mu^2
}.
\end{equation}
The supremum is attained by any nonzero rescaling
$\bh=c\bg$ with $c\neq0$ whenever $\ZFD(p,q)>0$. If $\ZFD(p,q)=0$, every admissible
$\bh$ attains value zero.
\end{proposition}
A proof can be found in Appendix \ref{app:prf_zfd_dual}.
Similarly, $\bh$ can be viewed as the witness function for ZFD, as it probes the discrepancy between $p$ and $q$ through the alignment $\langle\bh,\bg\rangle_\mu$. Crucially, this alignment equals $\ZFD(p,q)$ when $\bh=\bg$, while under $H_0$ it vanishes for \textit{any fixed} $\bh$ since $\bg\equiv\bzero$.
This linear alignment naturally leads to a two-stage \emph{learn-then-test} procedure based on sample splitting \citep{kubler2022witness}, thereby avoiding the squaring problem: under $H_0$, the held-out alignment remains centered at zero for any learned $\bh$, so estimation error in the witness does not induce a positive quadratic bias.
We also note that \citet{shlenskii2026midpoint} derive a quadratic variational form of the corresponding Midpoint Divergence, whereas our formulation takes a scale-invariant Rayleigh-quotient form that depends only on the witness direction and is naturally suited to the witness learning construction.

\subsection{Zero-Flow Two-Sample Test (ZF2ST)}

\textbf{Test statistic.}
Given paired samples $\calS:={\{(\bx_i,\by_i)\}}_{i=1}^n$ formed from $\calS_p$ and $\calS_q$, we partition $\calS$ into a training split $\calS^\rmtr$ of size $n_\rmtr$ and an independent held-out split $\calS^\rmte$ of size $n_\rmte$, where $n_\rmtr+n_\rmte=n$. \textit{Stage-I} learns a flexible witness $\bh$ on $\calS^\rmtr$, while \textit{Stage-II} evaluates the fixed witness on $\calS^\rmte$. Thus, witness learning determines testing power through its alignment with $\bg$, while sample splitting \textit{preserves} type-I error control. To construct the test statistic, consider the single score of the alignment
\begin{equation} \label{eq:population_align}
    \tau(\bh)
        :=
        \E\!\left[\fraks_{\bh}(X,Y,t,B)\right] \overset{(*)}{=}\langle\bh,\bg\rangle_\mu,
        \quad
        \fraks_{\bh}(X,Y,t,B)
        :=
        \langle
            \bh(\widetilde Z_t,t),
            Y-X
        \rangle,
\end{equation}
where $t \sim \operatorname{Unif}[0,1/2]$ and $B\sim\operatorname{Ber}(1/2)$ are sampled independently of $(X,Y)$ and $(*)$ holds by the tower property (\Cref{lem:tower_property}).
Importantly, estimating $\tau(\bh)$ does not require explicitly estimating $\bg$.
Moreover, when $\bh$ is fixed, its empirical mean over i.i.d. paired observations provides a simple unbiased estimator with an asymptotically normal distribution by the central limit theorem \citep{serfling2009approximation}.
\begin{proposition}[Asymptotic normality]
\label{prop:zf2st_asymptotic}
Fix a witness $\bh\in L^2(\mu;\R^d)$ and let $\sigma^2(\bh):=\var[\fraks_{\bh}(X,Y,t,B)]$.
Let ${(X_i,Y_i,t_i,B_i)}_{i=1}^n$ be i.i.d., with $X_i\sim p$ and $Y_i\sim q$ drawn independently, $t_i\sim\operatorname{Unif}[0,1/2]$, and $B_i\sim\operatorname{Ber}(1/2)$.
Denote by $\hat{\tau}(\bh)$ and
$\hat{\sigma}^2(\bh)$ the sample mean and sample variance of
$\{\fraks_{\bh}(X_i,Y_i,t_i,B_i)\}_{i=1}^n$.
If $0<\sigma^2(\bh)<\infty$, then, as $n\to\infty$,
\[
    \frac{
        \sqrt{n}\bigl(\hat{\tau}(\bh)-\tau(\bh)\bigr)
    }{
        \hat{\sigma}(\bh)
    }
    \xrightarrow{d}
    \mathcal N(0,1).
\]
\end{proposition}
A proof can be found in Appendix \ref{app:prf_zf2st_asymptotic}.
Since the witness $\bh$ is learned exclusively from the training split $\calS^{\rm tr}$ in Stage-I, it is \textit{fixed} when evaluated on the independent held-out split $\calS^{\rm te}$.
Applying \Cref{prop:zf2st_asymptotic} to the $n_\rmte$ held-out pairs, we denote the corresponding sample mean and standard deviation by $\hat\tau_{\rmte}(\bh)$ and $\hat\sigma_{\rmte}(\bh)$, respectively. The final ZF2ST statistic is the studentized empirical alignment
\begin{equation} \label{eq:zf_stats}
\zeta_{\rm ZF}
:=
\frac{
\sqrt{n_\rmte}
\hat\tau_{\rmte}(\bh)
}{
\hat\sigma_{\rmte}(\bh)
}.
\end{equation}

\paragraph{Testing Power and Consistency.}
For a fixed witness $\bh$, $\zeta_{\textrm{ZF}}$ is asymptotically $\mathcal{N}(0,1)$ under $H_0:p=q$.
Following \citet{kubler2022witness}, we use a one-sided test because Stage-I is designed to orient $\bh$ toward positive alignment with $\bg$.
A two-sided test discards
this directional information and therefore uses a more conservative rejection
threshold for positively aligned alternatives.
We thus reject $H_0$ whenever
$\zeta_{\text{ZF}} \geq \Phi^{-1}(1-\alpha)$.
\begin{proposition}[Power Approximation]
\label{prop:power_approx}
Fix witness $\bh$, and suppose
$0<\sigma^2(\bh)<\infty$.
For sufficiently large $n_\rmte$, the power of the one-sided level-$\alpha$ test admits the Gaussian approximation
\begin{equation}
\label{eq:power}
    \Pr\left(
        \zeta_{\rm ZF}
        \geq
        \Phi^{-1}(1-\alpha)
    \right)
    \approx
    1-\Phi\left(
        \Phi^{-1}(1-\alpha)
        -
        \sqrt{n_\rmte}\,
        {\rm SNR}(\bh)
    \right),
    \;
    {\rm SNR}(\bh)
    :=
    \frac{\tau(\bh)}{\sigma(\bh)}.
\end{equation}
Under any fixed alternative satisfying
$\tau(\bh)>0$, the power converges to one as
$n_\rmte\to\infty$.
\end{proposition}
A proof can be found in Appendix \ref{app:prf_power_approx}.
Here, ${\rm SNR}(\bh)$ denotes the signal-to-noise ratio, measuring the strength of the population alignment relative to the variability of the corresponding score.
\Cref{eq:power} shows that the power depends on $\sqrt{n_{\rm te}}{\rm SNR}(\bh)$, implying consistency for any positively aligned witness with $\tau(\bh)>0$, provided the score variance $\sigma^2(\bh)$ is finite and nonzero.

\subsection{Practical Implementation} \label{sec:practical_implementation}
\textbf{Stage I - Witness learning.} 
We consider two objectives for learning the witness $\bh$ from the training split $\calS^{\rm tr}$.
The first follows directly from \Cref{lem:gt_regform}: we learn the population field $\bg$ by minimizing the empirical flow-matching regression risk,
\begin{equation} \label{eq:learning_reg} 
\widehat{\bh}_{\rm reg} = \underset{\bh\in\calF}{\arg\min}\, \widehat{\calR}_{\rm reg}(\bh),
\quad
\widehat{\calR}_{\rm reg}(\bh) := \frac{1}{n_{\rm tr}} \sum_{i=1}^{n_{\rm tr}} \| Y_i-X_i-\bh(\widetilde Z_{t_i},t_i) \|^2,
\end{equation}
where $\calF$ denotes a neural-network function class, $t_i$, and $B_i$ are sampled on the fly for each pair, with $\widetilde Z_{t_i}=(1-B_i)Z_{t_i}+B_iZ_{1-t_i}$. This is implemented in the same manner as standard flow-matching training. We refer to this training objective as \textbf{ZF-Reg}.
Alternatively, following 
\citet{liu2020learning}, we directly optimize the empirical regularized SNR:
\begin{equation}
\label{eq:learning_snr}
    \widehat{\bh}_{\rm snr}
    =
    \underset{\bh\in\calF}{\arg\min}\;
    \widehat \calR_{\rm snr}(\bh),
    \quad
    \widehat \calR_{\rm snr}(\bh) = -
    \frac{\hat\tau_{\rmtr}(\bh)}{\hat \sigma_{\rmtr}^{\lambda}(\bh)}.
\end{equation}
Here, $\hat{\tau}(\bh)$ is the empirical mean of i.i.d. scores defined in \Cref{prop:zf2st_asymptotic} and $\hat{\sigma}_{\rmtr}^{\lambda}(\bh)$ denotes the empirical counterpart of the regularized standard deviation $\sigma^{\lambda}(\bh):=\sqrt{\sigma^2(\bh)+\lambda \|\bh\|^2_\mu}$ on $\calS^\rmtr$ with $\lambda>0$ providing ridge stabilization of the SNR denominator. We refer to this training objective as \textbf{ZF-SNR}.
For both ZF-Reg and ZF-SNR, we use dynamic pairing during training to better approximate the empirical product measure $\hat p\otimes\hat q$.
Implementation details are provided in \ref{app:detail_dyn_pair}.

\textbf{Stage II - Wild bootstrap calibration.}
For finite-sample calibration, we apply the wild bootstrap to the empirical held-out alignment $\hat{\tau}_{\rmte}(\bh)$ with Rademacher multipliers \citep{shao2010dependent,chwialkowski2014wild}.
Given a fixed witness $\bh$, for the $b$-th bootstrap replicate, we draw a Rademacher vector
$\boldsymbol{\epsilon}^{(b)}=[\epsilon_1^{(b)},\ldots,\epsilon_{n_{\rm te}}^{(b)}]^\top$, where the entries are independent and take values in $\{+1,-1\}$ with equal probability, and compute
\begin{equation}
\label{eq:wild_bootstrap}
    \hat{\tau}_{\rm te}^{(b)}(\bh)
    :=
    \frac{1}{n_{\rm te}}
    \sum_{i=1}^{n_{\rm te}}
    \epsilon_i^{(b)}
    \fraks_{\bh}(X_i,Y_i,t_i,B_i),
    \quad
    \hat p_{\rm val} := \frac{ 1+\sum_{b=1}^{n_{\rm perm}} \mathbf{1}\!\left\{ \hat{\tau}_{\rm te}^{(b)}(\bh) \geq \hat{\tau}_{\rm te}(\bh) \right\} }{ n_{\rm perm}+1 }.
\end{equation}
We reject $H_0$ whenever $\hat p_{\rm val}\leq\alpha$.
The resulting calibration controls the finite-sample Type-I error and is consistent under fixed alternatives with positive alignment, as formalized in Appendix~\ref{app:prf_finite_validity_consistency}.
In our paired setting, the wild bootstrap reduces to an exact sign-flip randomization \citep{hemerik2020robust}: under $H_0$, the transformation $(X,Y,B)\mapsto(Y,X,1-B)$ preserves the joint distribution while reversing the score $\fraks_{\bh}(Y,X,t,1-B)= -\fraks_{\bh}(X,Y,t,B)$.
Therefore, the null distribution can be generated entirely from the pre-computed held-out scores.
As in witness-based tests \citep{kubler2022witness,kubler2022automl}, these scores are computed in a \textit{single} evaluation pass over the held-out data and reused for all subsequent randomization replicates, enabling fast calibration with only sign flips and no additional witness evaluations.

\textbf{Variance reduction.}
To further improve testing power, we introduce a variance-reduced ZF-SNR objective that preserves the population alignment $\tau(\bh)$ while reducing variance induced by auxiliary variables.
As shown in \Cref{prop:power_approx}, holding the alignment unchanged, a smaller score variance increases the SNR and consequently the asymptotic testing power.
It can be shown that by the law of
total variance,
\begin{equation}
\label{eq:score_var_decomp}
\sigma^2(\bh)
=
\underbrace{\sigma^2_{X,Y}}_{\text{pair $X,Y$}}
+
\underbrace{\E_{X,Y}[\sigma^2_{t\mid X,Y}]}_{\text{time $t$}}
+
\underbrace{\E_{X,Y,t}[\sigma^2_{B\mid X,Y,t}]}_{\text{branch $B$}} .
\end{equation}
where $\sigma^2_{X,Y} := \var_{X,Y}(\E_{t,B}[\fraks_{\bh}| X,Y])$, $\sigma^2_{t|X,Y} := \var_t(\E_B[\fraks_{\bh}| X,Y,t]\mid X,Y)$ and $\sigma^2_{B\mid X,Y,t}
:=
\var_B(\fraks_{\bh}| X,Y,t)$.
The last two terms arise solely from sampling $t$ and $B$. We average out $B$ exactly and reduce the remaining time-sampling variance using $K$ i.i.d.\ samples $t_{1:K}\sim\operatorname{Unif}[0,1/2]$:
\begin{equation}
\label{eq:branch_avg_score}
    \fraks_{\bh,K}(X,Y,t_{1:K})
    :=
    \frac{1}{K}
    \sum_{k=1}^K
    \bar{\fraks}_{\bh}(t_k).
    \quad
    \bar{\fraks}_{\bh}(t)
    :=
    \frac{1}{2}
    \left\langle
        \bh(Z_t,t)+\bh(Z_{1-t},t),
        Y-X
    \right\rangle.
\end{equation}
Here, $\bar{\fraks}_{\bh}(t)=\E_B[\fraks_{\bh}\mid X,Y,t]$, with the dependence on $(X,Y)$ suppressed for brevity.
Letting $\sigma_K^2(\bh):=\var(\fraks_{\bh,K})$, we have
$\sigma_K^2(\bh)=\sigma_{X,Y}^2+\tfrac{1}{K}\E_{X,Y}[\sigma_{t\mid X,Y}^2]\leq\sigma^2(\bh)$.
Moreover, $\E[\fraks_{\bh,K}]=\tau(\bh)$ for any $K\geq1$, so the population alignment is preserved while the auxiliary variance is reduced.
Accordingly, during training, we replace the ZF-SNR objective $\calR_{\rm snr}(\bh):=-\tau(\bh)/\sigma^\lambda(\bh)$ with its variance-reduced counterpart $\calR_{{\rm snr},K}(\bh):=-\tau(\bh)/\sigma_K^\lambda(\bh)$, and optimize the corresponding empirical objective $\widehat{\calR}_{{\rm snr},K}$ on $\calS^{\rm tr}$. An explicit form is provided in Appendix~\ref{app:explicit_emp_snr_vr}.

At test time, ZF-SNR uses the same $K$ as in training to align the training objective with the test statistic.
Since $\sigma_K^2(\bh)\leq\sigma^2(\bh)$ while the alignment is unchanged, any fixed witness with positive alignment attains no lower asymptotic power when evaluated using $\fraks_{\bh,K}$.
Moreover, $\fraks_{\bh,K}$ preserves swap antisymmetry, so wild-bootstrap validity remains unchanged.
We therefore also use $\fraks_{\bh,K}$ for ZF-Reg at test time. This requires $2K$ witness evaluations during testing for both methods. We summarize the full ZF2ST procedure in \Cref{alg:zf2st}, with a detailed complexity analysis provided in Appendix \ref{app:complexity}.

\section{Experiments} \label{sec:experiments}

We compare ZF2ST against a broad range of kernel-based and witness-learning two-sample tests, spanning both sample-splitting and non-splitting approaches.
Among sample-splitting methods, we consider deep-kernel MMD (MMD-D; \citealp{liu2020learning}), classifier two-sample tests based on sign (C2ST-S; \citealp{lopez2016revisiting}) and linear scores (C2ST-L; \citealp{cheng2022classification}), the KFDA witness test (KFDA-WiTS; \citealp{kubler2022witness}) and its more flexible AutoML-based extension (AutoML; \citealp{kubler2022automl}).
We further include strong non-splitting kernel tests, namely MMDagg (\citealp{schrab2023mmd}) and MMD-FUSE (\citealp{biggs2023mmd}).

For sample-splitting methods, we run 10 independent trainings, each followed by 100 independent held-out test trials. We use $n=n_{\rm tr}+n_{\rm te}$ with $n_{\rm tr}=n_{\rm te}$. Sample-splitting methods train on $n_{\rm tr}$ and test on $n_{\rm te}$, while non-splitting methods use all $n$ samples over 100 independent trials. 
For non-image experiments, ZF2ST \textit{matches} the MLP architecture of MMD-D; for image experiments, ZF2ST uses a \textit{substantially smaller} network, with architecture and capacity comparisons provided in Appendix~\ref{app:ablation_nn}.
For ZF-SNR, we use its variance-reduced version and set $(K,\lambda)=(3,10^{-3})$ for non-image experiments and $(2,10^{-2})$ for image experiments.
ZF2ST uses wild bootstrap calibration in all experiments.
Further experiment details are provided in Appendix~\ref{app:experiment_detail}.

\textbf{Synthetic benchmarks.}
Following \citet{liu2020learning}, we consider the \texttt{Blobs} dataset with multiple Gaussian components of different covariances and the high-dimensional Gaussian mixture (\texttt{HDGM}) with differences only in the first two dimensions.
For \texttt{HDGM}, we vary either the sample size with $d=10$ fixed or the ambient dimension with $n=2000$ fixed.
As shown in the top row of \Cref{fig:power_experiments}, ZF-SNR rapidly gains power on \texttt{Blobs}, while ZF-Reg shows a more modest improvement as the sample size increases.
On \texttt{HDGM}, both ZF-SNR and ZF-Reg perform strongly across two settings. ZF-Reg is particularly robust as the ambient dimension increases and also performs best in the larger-sample regime, suggesting that regression-based witness learning may be especially well suited to this type of sparse, high-dimensional alternative.

\begin{figure}[t]
    \centering
    \includegraphics[width=\linewidth]{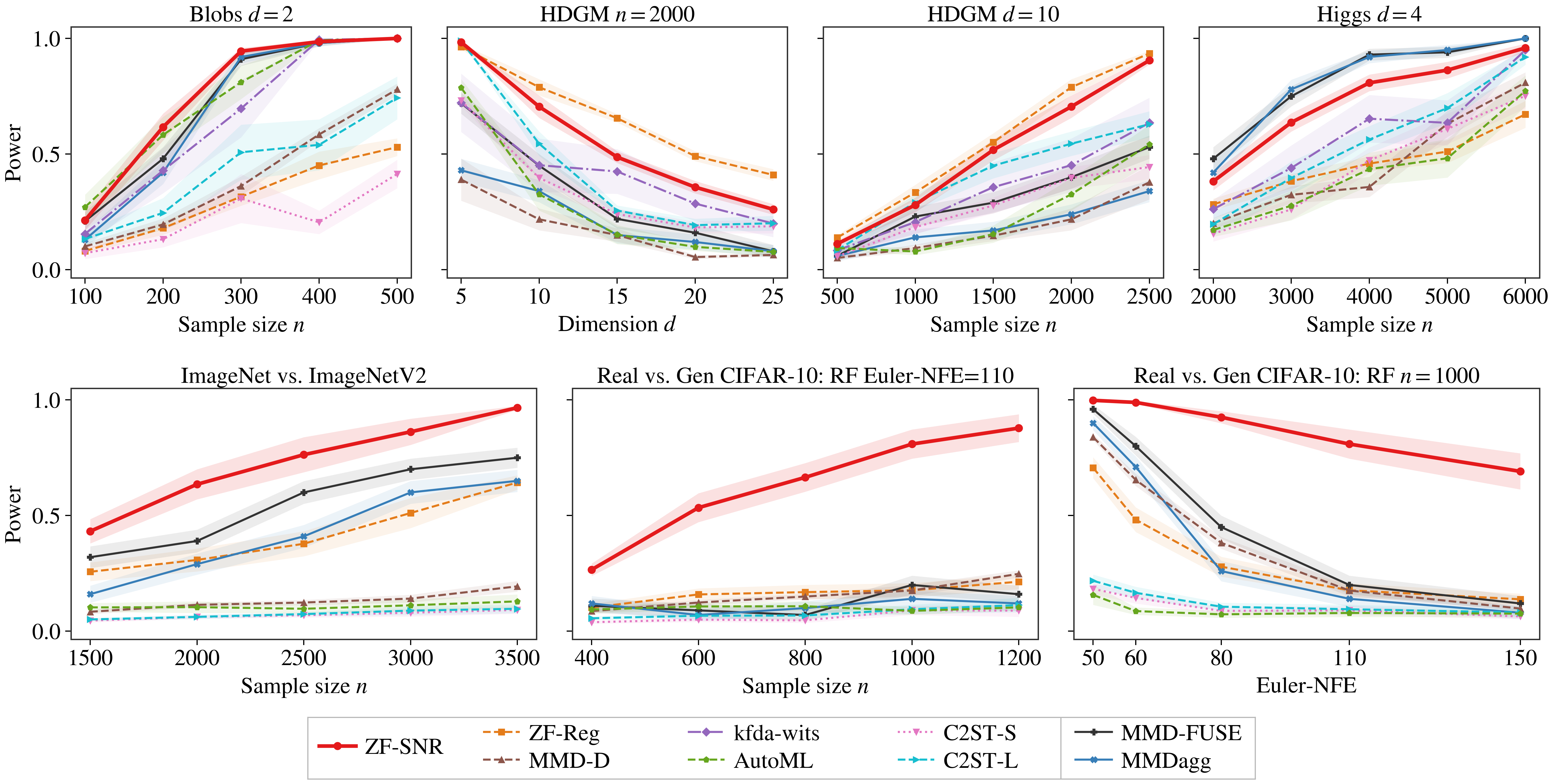}
    \caption{Power experiments at level $\alpha=0.05$. The top row reports \texttt{Blobs}, \texttt{HDGM}, and \texttt{Higgs}, varying either sample size $n$ or dimension $d$ as indicated. The bottom row reports ImageNet vs.\ ImageNetV2 and real vs.\ RF-generated CIFAR-10, varying either the sample size or the NFE of the Euler solver. All images are resized to $3\times64\times64$. KFDA-WiTS is evaluated only on the non-image data. Lines show mean test power, with shaded regions indicating $\pm1$ standard error. Corresponding Type-I error results are provided in \Cref{fig:type1}. \vspace{-2em}}
    \label{fig:power_experiments}
\end{figure}

\textbf{Higgs dataset \citep{baldi2014searching}.}
Following \citet{liu2020learning}, we consider the \texttt{Higgs} dataset, where the signal is known to be relatively weak \citep{chwialkowski2015fast}.
As shown in the top-right panel of \Cref{fig:power_experiments}, ZF-SNR consistently outperforms the other sample-splitting methods.
It also remains competitive with the non-splitting kernel tests, despite using only held-out samples for testing, while its improvement over ZF-Reg indicates an additional benefit from the SNR-based training objective.

\textbf{Real-data distribution shift.}
We first consider CIFAR-10 \citep{krizhevsky2009learning} vs. CIFAR-10.1 \citep{recht2019imagenet} where CIFAR-10.1 was independently recollected to reproduce the original CIFAR-10 image classification task while using a separate data-collection process.
This provides a natural real-data distribution-shift benchmark and has been studied in two-sample testing by \citet{liu2020learning} and \citet{biggs2023mmd}. 
We further consider ImageNet \citep{deng2009imagenet} vs. ImageNetV2 \citep{recht2019imagenet},
which applies the same replication principle to the more challenging ImageNet classification task using newly collected test images.
\citet{liu2020learning} identified this comparison as a more challenging extension.

On CIFAR-10 vs. CIFAR-10.1, \Cref{tab:cifar10-cifar101} shows that both ZF variants achieve power close to one and outperform all the competing methods.
Together, these results provide strong statistical evidence of distribution shifts in both benchmarks and demonstrate that ZF2ST remains effective on high-dimensional natural images.

\textbf{CIFAR-10 with modern generative models.}
Comparing generated samples with observed data has long been a natural application of two-sample testing for evaluating generative models \citep{sutherland2016generative,lopez2016revisiting}. 
Since generative models are explicitly trained to make their output distribution $q$ closely approximate the data distribution $p$, distinguishing the two distributions can be particularly challenging when the remaining differences are subtle and concentrated in complex high-dimensional image statistics.
Recent advances in generative modeling have dramatically improved the quality of synthesized images \citep{dhariwal2021diffusion,rombach2022high}, further narrowing the gap between generated and real data.
We therefore evaluate ZF2ST across several recent generative models to assess its ability to detect these residual distributional differences.

We first compare real CIFAR-10 with samples generated by a 1-Rectified Flow (RF, \citealp{liu2022flow}) model using an Euler solver.
The \underline{N}umber of \underline{F}unction \underline{E}valuations (NFE) determines the number of Euler steps used to solve the learned ODE when generating each sample.
As noted by \citet{liu2022flow}, increasing the Euler NFE reduces discretization error and improves sampling quality. We therefore consider two settings: varying the sample size $n$ with $\text{NFE}=110$ fixed, and varying the Euler NFE with $n=1000$ fixed.
As shown in the bottom-middle and bottom-right panels of \Cref{fig:power_experiments}, ZF-SNR exhibits a striking advantage over all baselines, with power increasing rapidly with sample size and remaining substantially higher as the Euler NFE increases.

We further evaluate ZF2ST across several representative generative model families to examine whether its strong performance extends beyond a particular generation mechanism.
The benchmark includes multi-step generative models from both diffusion and flow paradigms, represented by ScoreSDE \citep{song2020score} and 1-Rectified Flow with a higher-accuracy RK45 solver \citep{liu2022flow}, respectively.
We also consider the adversarially trained StyleGAN-XL (SGXL, \citealp{sauer2022stylegan}) and MeanFlow \citep{geng2026mean}, a recent high-quality one-step generative model. For each model, we use the official pretrained checkpoint to generate 10,000 samples, matching the size of the CIFAR-10 test set, and perform two-sample testing with $n=1000$ samples from each distribution.

As shown in \Cref{tab:genmodel_power}, ZF-SNR achieves the highest power on three of the four generators. On MeanFlow, non-splitting MMD methods perform better, while ZF-SNR still achieves substantially higher power than the remaining sample-splitting baselines. Overall, these results suggest that the strong performance of ZF-SNR generalizes across diverse generative mechanisms rather than being tied to a particular model family or numerical solver.
We also report clean FID (cFID, \citealp{parmar2021cleanfid}) computed using 10,000 generated samples, providing a more standardized and fair comparison, with lower values indicating better generative quality.
Test power in \Cref{tab:genmodel_power} generally increases with cFID, consistent with lower-quality generated samples being easier to distinguish from real data. However, SGXL is a notable exception: despite having the lowest cFID, all tests achieve higher power than on ScoreSDE, indicating that cFID does not fully capture the distributional differences detected by two-sample tests.

\begin{table}[t]
    \centering

    \begin{subtable}[t]{0.29\linewidth}
        \centering
        \small
        \caption{CIFAR-10 vs.\ CIFAR-10.1.}
        \label{tab:cifar10-cifar101}
        \begin{tabular}{lc}
            \toprule
            \raisebox{-0.35em}{Tests}
            &
            \raisebox{-0.35em}{Power (\%)} \\[0.8em]
            \midrule
            ZF-SNR   & 99.5 {\scriptsize $\pm$ 0.27} \\
            ZF-Reg   & \textbf{99.8} {\scriptsize $\pm$ 0.13} \\
            MMD-FUSE & 94.0 {\scriptsize $\pm$ 2.39} \\
            MMDagg   & 86.0 {\scriptsize $\pm$ 3.49} \\
            MMD-D    & 83.3 {\scriptsize $\pm$ 3.65} \\
            AutoML   & 78.4 {\scriptsize $\pm$ 7.39} \\
            C2ST-L   & 43.3 {\scriptsize $\pm$ 8.17} \\
            C2ST-S   & 40.1 {\scriptsize $\pm$ 8.37} \\
            \bottomrule
        \end{tabular}
    \end{subtable}
    \hfill
\begin{subtable}[t]{0.69\linewidth}
    \centering
    \small
    \caption{Real vs.\ generated CIFAR-10 across different generative models.}
    \label{tab:genmodel_power}
    \setlength{\tabcolsep}{3.5pt}
    \begin{tabular}{lcccc}
        \toprule
        \raisebox{0.35em}{Tests \scriptsize{($n=1000$)}}
        & \shortstack{ScoreSDE\\{\scriptsize (cFID$_{10\mathrm{k}}$=6.48)}}
        & \shortstack{RF RK45\\{\scriptsize (cFID$_{10\mathrm{k}}$=7.38)}}
        & \shortstack{MeanFlow\\{\scriptsize (cFID$_{10\mathrm{k}}$=7.48)}}
        & \shortstack{SGXL\\{\scriptsize (cFID$_{10\mathrm{k}}$=6.29)}} \\
        \midrule
        ZF-SNR
        & \textbf{32.3} {\scriptsize $\pm$ 5.51}
        & \textbf{40.3} {\scriptsize $\pm$ 4.78}
        & 59.6 {\scriptsize $\pm$ 3.71}
        & \textbf{36.3} {\scriptsize $\pm$ 5.53} \\

        ZF-Reg
        & 10.0 {\scriptsize $\pm$ 3.32}
        & 16.6 {\scriptsize $\pm$ 3.54}
        & 35.0 {\scriptsize $\pm$ 5.37}
        & 23.2 {\scriptsize $\pm$ 3.32} \\

        MMD-FUSE
        & 17.0 {\scriptsize $\pm$ 3.78}
        & 27.0 {\scriptsize $\pm$ 4.46}
        & \textbf{79.0} {\scriptsize $\pm$ 4.09}
        & 18.0 {\scriptsize $\pm$ 3.86} \\

        MMDagg
        & 15.0 {\scriptsize $\pm$ 3.59}
        & 15.0 {\scriptsize $\pm$ 3.59}
        & 75.0 {\scriptsize $\pm$ 4.35}
        & 22.0 {\scriptsize $\pm$ 4.16} \\

        MMD-D
        & 9.1 {\scriptsize $\pm$ 1.26}
        & 8.2 {\scriptsize $\pm$ 0.59}
        & 28.1 {\scriptsize $\pm$ 2.90}
        & 14.7 {\scriptsize $\pm$ 1.34} \\

        AutoML
        & 11.4 {\scriptsize $\pm$ 1.71}
        & 9.1 {\scriptsize $\pm$ 1.85}
        & 16.1 {\scriptsize $\pm$ 5.17}
        & 26.8 {\scriptsize $\pm$ 2.66} \\

        C2ST-L
        & 8.0 {\scriptsize $\pm$ 1.65}
        & 6.3 {\scriptsize $\pm$ 1.15}
        & 10.0 {\scriptsize $\pm$ 0.92}
        & 14.2 {\scriptsize $\pm$ 2.01} \\

        C2ST-S
        & 6.9 {\scriptsize $\pm$ 1.43}
        & 5.7 {\scriptsize $\pm$ 0.99}
        & 8.1 {\scriptsize $\pm$ 0.91}
        & 12.6 {\scriptsize $\pm$ 1.86} \\
        \bottomrule
    \end{tabular}
\end{subtable}
    \caption{Power experiments at level $\alpha=0.05$ on \textbf{(a)} CIFAR-10 vs. CIFAR-10.1  \textbf{(b)} real CIFAR-10 vs. samples generated by four different models, with $n=1000$ samples per distribution. $\mathrm{cFID}_{10\mathrm{k}}$ denotes clean FID computed using 10,000 generated samples. All images are resized to $3\times64\times64$. Test power (\%) is reported as mean $\pm$ standard error, with the best result shown in bold.}
    \label{tab:image_results}
\end{table}

\section{Conclusion}
In this work, we proposed ZF2ST, a flow-based framework for two-sample testing built on the zero-flow characterization of distributional equality. We developed a principled discrepancy and testing procedure and demonstrated competitive performance across synthetic, tabular, and image benchmarks, with ZF-Reg and ZF-SNR showing complementary strengths and particularly strong image performance for SNR optimization. A current limitation is the reliance on sample splitting, which reduces data efficiency, while a complete finite-sample theory for data-dependent learned witnesses remains open.

Future work could develop more data-efficient testing procedures that relax or avoid sample splitting, establish finite-sample guarantees that account for witness estimation, and extend ZF2ST to other hypothesis testing problems.

\subsection*{AI Use Statement}

In this work, we used generative AI tools to assist with drafting high-level section structures, polishing writing, retrieving and discovering related work, and checking mathematical arguments and presentation. All relevant literature was read and verified by the authors to ensure accurate representation. The core research ideas, theorem statements, and proofs were developed and written by the authors; generative AI tools were not used to generate the main research ideas or to prove mathematical claims.

Our \textbf{CIFAR-10 with modern generative models} experiments in \Cref{sec:experiments} use synthetic image samples produced by the pretrained generative models being evaluated. These samples were not generated by an LLM or AI assistant; the corresponding models, checkpoints, and sampling procedures are fully disclosed in Appendix~\ref{app:ds_preprocess}. All AI-assisted content was reviewed and verified by the authors.

The authors take full responsibility for the final content of the paper, including all text, claims, proofs, code, and experimental results produced with the aid of generative AI.





\subsection*{Ethics Statement}

This work develops a general statistical methodology for two-sample testing and does not involve the collection of new human-subject data or personally identifiable information. All empirical evaluations use publicly available benchmark datasets and publicly released pretrained generative models, with data sources and model checkpoints documented in Appendix~\ref{app:experiment_detail}. We do not release any new dataset. As with general-purpose statistical testing tools, ZF2ST may be applied to domains involving sensitive data; such applications should follow the relevant requirements regarding data access, privacy, consent, and responsible interpretation of statistical results. We are not aware of other direct ethical concerns specific to the methodology introduced in this work.




\section*{Reproducibility Statement}
We provide detailed material to support reproducibility of both the theoretical and empirical results. Complete proofs of the main theoretical results, including the measure-theoretic formulation and the validity of the wild bootstrap calibration, are given in Appendix~\ref{app:all_prfs}. Appendix~\ref{app:implementation_details} provides the full ZF2ST algorithm together with implementation details for dynamic pairing, the variance-reduced objective, and computational complexity. Appendix~\ref{app:experiment_detail} specifies the experimental protocols, datasets and preprocessing, model architectures, baseline implementations, hyperparameters, Type-I error results, and runtime measurements. Additional ablation studies are provided in Appendix~\ref{app:ablation_all}, and visualizations of the datasets and generated samples are included in Appendix~\ref{app:data_visualization}. For experiments involving pretrained generative models, Appendix~\ref{app:experiment_detail} also provides the corresponding official code repositories, checkpoints, and sampling configurations.





\subsubsection*{Acknowledgments}
We thank Wittawat Jitkrittum and Wenkai Xu for their valuable feedback, and Eshant English for helpful discussions.
LW was partially supported by JSPS KAKENHI (24K02905).
TS was partially supported by JST CREST (PMJCR2015) and JSPS KAKENHI (25H01107).
This research is supported by the National Research Foundation, Singapore and the Ministry of Digital Development and Information under the AI Visiting Professorship Programme (award number AIVP-2024-004). Any opinions, findings and conclusions or recommendations expressed in this material are those of the author(s) and do not reflect the views of National Research Foundation, Singapore and the Ministry of Digital Development and Information.

\bibliography{iclr2027_conference}

@inproceedings{liu2020learning,
  title={Learning deep kernels for non-parametric two-sample tests},
  author={Liu, Feng and Xu, Wenkai and Lu, Jie and Zhang, Guangquan and Gretton, Arthur and Sutherland, Danica J},
  booktitle={International conference on machine learning},
  pages={6316--6326},
  year={2020},
  organization={PMLR}
}

@article{lopez2016revisiting,
  title={Revisiting classifier two-sample tests},
  author={Lopez-Paz, David and Oquab, Maxime},
  journal={arXiv preprint arXiv:1610.06545},
  year={2016}
}

@article{sugiyama2011least,
  title={Least-squares two-sample test},
  author={Sugiyama, Masashi and Suzuki, Taiji and Itoh, Yuta and Kanamori, Takafumi and Kimura, Manabu},
  journal={Neural networks},
  volume={24},
  number={7},
  pages={735--751},
  year={2011},
  publisher={Elsevier}
}

@article{gretton2012kernel,
  title={A kernel two-sample test},
  author={Gretton, Arthur and Borgwardt, Karsten M and Rasch, Malte J and Sch{\"o}lkopf, Bernhard and Smola, Alexander},
  journal={The journal of machine learning research},
  volume={13},
  number={1},
  pages={723--773},
  year={2012},
  publisher={JMLR. org}
}

@inproceedings{wang2026zero,
      title={Zero-Flow Encoders}, 
      author={Yakun Wang and Leyang Wang and Song Liu and Taiji Suzuki},
      booktitle={International Conference on Machine Learning},
      year={2026}
}

@inproceedings{kubler2022witness,
  title={A witness two-sample test},
  author={K{\"u}bler, Jonas M and Jitkrittum, Wittawat and Sch{\"o}lkopf, Bernhard and Muandet, Krikamol},
  booktitle={International Conference on Artificial Intelligence and Statistics},
  pages={1403--1419},
  year={2022},
  organization={PMLR}
}

@article{sutherland2016generative,
  title={Generative models and model criticism via optimized maximum mean discrepancy},
  author={Sutherland, Danica J and Tung, Hsiao-Yu and Strathmann, Heiko and De, Soumyajit and Ramdas, Aaditya and Smola, Alex and Gretton, Arthur},
  journal={arXiv preprint arXiv:1611.04488},
  year={2016}
}

@book{villani2009optimal,
  title={Optimal transport: old and new},
  author={Villani, C{\'e}dric and others},
  volume={338},
  year={2009},
  publisher={Springer}
}

@article{jitkrittum2016interpretable,
  title={Interpretable distribution features with maximum testing power},
  author={Jitkrittum, Wittawat and Szab{\'o}, Zolt{\'a}n and Chwialkowski, Kacper P and Gretton, Arthur},
  journal={Advances in Neural Information Processing Systems},
  volume={29},
  year={2016}
}

@article{cheng2022classification,
  title={Classification logit two-sample testing by neural networks for differentiating near manifold densities},
  author={Cheng, Xiuyuan and Cloninger, Alexander},
  journal={IEEE Transactions on Information Theory},
  volume={68},
  number={10},
  pages={6631--6662},
  year={2022},
  publisher={IEEE}
}

@article{liu2022flow,
  title={Flow straight and fast: Learning to generate and transfer data with rectified flow},
  author={Liu, Xingchao and Gong, Chengyue and Liu, Qiang},
  journal={arXiv preprint arXiv:2209.03003},
  year={2022}
}

@book{good2013permutation,
  title={Permutation tests: a practical guide to resampling methods for testing hypotheses},
  author={Good, Phillip},
  year={2013},
  publisher={Springer Science \& Business Media}
}

@article{efron1985bootstrap,
  title={The bootstrap method for assessing statistical accuracy},
  author={Efron, Bradley and Tibshirani, Robert},
  journal={Behaviormetrika},
  volume={12},
  number={17},
  pages={1--35},
  year={1985},
  publisher={Springer}
}

@inproceedings{lipmanflow,
  title={Flow Matching for Generative Modeling},
  author={Lipman, Yaron and Chen, Ricky TQ and Ben-Hamu, Heli and Nickel, Maximilian and Le, Matthew},
  booktitle={International Conference on Learning Representations},
  year={2023}
}

@book{serfling2009approximation,
  title={Approximation theorems of mathematical statistics},
  author={Serfling, Robert J},
  year={2009},
  publisher={John Wiley \& Sons}
}

@article{albergo2025stochastic,
  title={Stochastic interpolants: A unifying framework for flows and diffusions},
  author={Albergo, Michael and Boffi, Nicholas M and Vanden-Eijnden, Eric},
  journal={Journal of Machine Learning Research},
  volume={26},
  number={209},
  pages={1--80},
  year={2025}
}

@article{hemerik2018exact,
  title={Exact testing with random permutations},
  author={Hemerik, Jesse and Goeman, Jelle},
  journal={Test},
  volume={27},
  number={4},
  pages={811--825},
  year={2018},
  publisher={Springer}
}

@inproceedings{ramdas2015decreasing,
  title={On the decreasing power of kernel and distance based nonparametric hypothesis tests in high dimensions},
  author={Ramdas, Aaditya and Reddi, Sashank Jakkam and P{\'o}czos, Barnab{\'a}s and Singh, Aarti and Wasserman, Larry},
  booktitle={Proceedings of the AAAI Conference on Artificial Intelligence},
  volume={29},
  number={1},
  year={2015}
}

@article{rabanser2019failing,
  title={Failing loudly: An empirical study of methods for detecting dataset shift},
  author={Rabanser, Stephan and G{\"u}nnemann, Stephan and Lipton, Zachary},
  journal={Advances in Neural Information Processing Systems},
  volume={32},
  year={2019}
}

@article{shlenskii2026midpoint,
  title={Midpoint Generative Models},
  author={Shlenskii, Daniil and Gushchin, Nikita and Novitskiy, Lev and Dylov, Dmitry V and Korotin, Alexander},
  journal={arXiv preprint arXiv:2605.29920},
  year={2026}
}

@article{butakov2025fmmi,
  title={Fmmi: Flow matching mutual information estimation},
  author={Butakov, Ivan and Semenenko, Alexander and Kirova, Valeriya and Frolov, Alexey and Oseledets, Ivan},
  journal={arXiv preprint arXiv:2511.08552},
  year={2025}
}

@article{yu2025missing,
  title={Missing data imputation by reducing mutual information with rectified flows},
  author={Yu, Jiahao and Ying, Qizhen and Wang, Leyang and Jiang, Ziyue and Liu, Song},
  journal={Advances in Neural Information Processing Systems},
  volume={38},
  pages={80324--80352},
  year={2025}
}

@article{sriperumbudur2011universality,
  title={Universality, Characteristic Kernels and RKHS Embedding of Measures.},
  author={Sriperumbudur, Bharath K and Fukumizu, Kenji and Lanckriet, Gert RG},
  journal={Journal of Machine Learning Research},
  volume={12},
  number={7},
  year={2011}
}

@book{ambrosio2005gradient,
  title={Gradient flows: in metric spaces and in the space of probability measures},
  author={Ambrosio, Luigi and Gigli, Nicola and Savar{\'e}, Giuseppe},
  year={2005},
  publisher={Springer}
}

@article{kubler2022automl,
  title={Automl two-sample test},
  author={K{\"u}bler, Jonas M and Stimper, Vincent and Buchholz, Simon and Muandet, Krikamol and Sch{\"o}lkopf, Bernhard},
  journal={Advances in Neural Information Processing Systems},
  volume={35},
  pages={15929--15941},
  year={2022}
}

@article{anderson1994two,
  title={Two-sample test statistics for measuring discrepancies between two multivariate probability density functions using kernel-based density estimates},
  author={Anderson, Niall H and Hall, Peter and Titterington, D Michael},
  journal={Journal of Multivariate Analysis},
  volume={50},
  number={1},
  pages={41--54},
  year={1994},
  publisher={Elsevier}
}

@article{hall1984central,
  title={Central limit theorem for integrated square error of multivariate nonparametric density estimators},
  author={Hall, Peter},
  journal={Journal of multivariate analysis},
  volume={14},
  number={1},
  pages={1--16},
  year={1984},
  publisher={Elsevier}
}

@article{biggs2023mmd,
  title={MMD-FUSE: Learning and combining kernels for two-sample testing without data splitting},
  author={Biggs, Felix and Schrab, Antonin and Gretton, Arthur},
  journal={Advances in Neural Information Processing Systems},
  volume={36},
  pages={75151--75188},
  year={2023}
}

@article{zhou2026dual,
  title={DUAL: Learning diverse kernels for aggregated two-sample and independence testing},
  author={Zhou, Zhijian and Tian, Xunye and Peng, Liuhua and Lei, Chao and Schrab, Antonin and Sutherland, Danica J and Liu, Feng},
  journal={Advances in Neural Information Processing Systems},
  volume={38},
  pages={130661--130695},
  year={2025}
}

@article{schrab2023mmd,
  title={MMD aggregated two-sample test},
  author={Schrab, Antonin and Kim, Ilmun and Albert, M{\'e}lisande and Laurent, B{\'e}atrice and Guedj, Benjamin and Gretton, Arthur},
  journal={Journal of Machine Learning Research},
  volume={24},
  number={194},
  pages={1--81},
  year={2023}
}

@article{shao2010dependent,
  title={The dependent wild bootstrap},
  author={Shao, Xiaofeng},
  journal={Journal of the American Statistical Association},
  volume={105},
  number={489},
  pages={218--235},
  year={2010},
  publisher={Taylor \& Francis}
}

@article{chwialkowski2014wild,
  title={A wild bootstrap for degenerate kernel tests},
  author={Chwialkowski, Kacper and Sejdinovic, Dino and Gretton, Arthur},
  journal={Advances in neural information processing systems},
  volume={27},
  year={2014}
}

@article{baldi2014searching,
  title={Searching for exotic particles in high-energy physics with deep learning},
  author={Baldi, Pierre and Sadowski, Peter and Whiteson, Daniel},
  journal={Nature communications},
  volume={5},
  number={1},
  pages={4308},
  year={2014},
  publisher={Nature Publishing Group UK London}
}

@inproceedings{recht2019imagenet,
  title={Do imagenet classifiers generalize to imagenet?},
  author={Recht, Benjamin and Roelofs, Rebecca and Schmidt, Ludwig and Shankar, Vaishaal},
  booktitle={International conference on machine learning},
  pages={5389--5400},
  year={2019},
  organization={PMLR}
}

@article{chwialkowski2015fast,
  title={Fast two-sample testing with analytic representations of probability measures},
  author={Chwialkowski, Kacper P and Ramdas, Aaditya and Sejdinovic, Dino and Gretton, Arthur},
  journal={Advances in Neural Information Processing Systems},
  volume={28},
  year={2015}
}

@article{smirnov1939estimation,
  title={On the estimation of the discrepancy between empirical curves of distribution for two independent samples},
  author={Smirnov, Nikolai V and others},
  journal={Bull. Math. Univ. Moscou},
  volume={2},
  number={2},
  pages={3--14},
  year={1939}
}

@article{krizhevsky2009learning,
  title={Learning multiple layers of features from tiny images},
  author={Krizhevsky, Alex and Hinton, Geoffrey and others},
  year={2009},
  publisher={Toronto, ON, Canada}
}

@inproceedings{deng2009imagenet,
  title={Imagenet: A large-scale hierarchical image database},
  author={Deng, Jia and Dong, Wei and Socher, Richard and Li, Li-Jia and Li, Kai and Fei-Fei, Li},
  booktitle={2009 IEEE conference on computer vision and pattern recognition},
  pages={248--255},
  year={2009},
  organization={Ieee}
}

@article{dhariwal2021diffusion,
  title={Diffusion models beat gans on image synthesis},
  author={Dhariwal, Prafulla and Nichol, Alexander},
  journal={Advances in neural information processing systems},
  volume={34},
  pages={8780--8794},
  year={2021}
}

@inproceedings{rombach2022high,
  title={High-resolution image synthesis with latent diffusion models},
  author={Rombach, Robin and Blattmann, Andreas and Lorenz, Dominik and Esser, Patrick and Ommer, Bj{\"o}rn},
  booktitle={2022 IEEE/CVF conference on computer vision and pattern recognition (CVPR)},
  pages={10674--10685},
  year={2022},
  organization={ieee}
}

@article{song2020score,
  title={Score-based generative modeling through stochastic differential equations},
  author={Song, Yang and Sohl-Dickstein, Jascha and Kingma, Diederik P and Kumar, Abhishek and Ermon, Stefano and Poole, Ben},
  journal={arXiv preprint arXiv:2011.13456},
  year={2020}
}

@article{geng2026mean,
  title={Mean flows for one-step generative modeling},
  author={Geng, Zhengyang and Deng, Mingyang and Bai, Xingjian and Kolter, Zico and He, Kaiming},
  journal={Advances in Neural Information Processing Systems},
  volume={38},
  pages={75460--75482},
  year={2025}
}

@inproceedings{sauer2022stylegan,
  title={Stylegan-xl: Scaling stylegan to large diverse datasets},
  author={Sauer, Axel and Schwarz, Katja and Geiger, Andreas},
  booktitle={ACM SIGGRAPH 2022 conference proceedings},
  pages={1--10},
  year={2022}
}

@inproceedings{parmar2021cleanfid,
  title={On Aliased Resizing and Surprising Subtleties in GAN Evaluation},
  author={Parmar, Gaurav and Zhang, Richard and Zhu, Jun-Yan},
  booktitle={CVPR},
  year={2022}
}

@article{borgwardt2006integrating,
  title={Integrating structured biological data by kernel maximum mean discrepancy},
  author={Borgwardt, Karsten M and Gretton, Arthur and Rasch, Malte J and Kriegel, Hans-Peter and Sch{\"o}lkopf, Bernhard and Smola, Alex J},
  journal={Bioinformatics},
  volume={22},
  number={14},
  pages={e49--e57},
  year={2006},
  publisher={Oxford University Press}
}

@article{fasano1987multidimensional,
  title={A multidimensional version of the Kolmogorov--Smirnov test},
  author={Fasano, Giovanni and Franceschini, Alberto},
  journal={Monthly Notices of the Royal Astronomical Society},
  volume={225},
  number={1},
  pages={155--170},
  year={1987},
  publisher={The Royal Astronomical Society}
}

@article{muandet2017kernel,
  title={Kernel mean embedding of distributions: A review and beyond},
  author={Muandet, Krikamol and Fukumizu, Kenji and Sriperumbudur, Bharath and Sch{\"o}lkopf, Bernhard},
  journal={Foundations and Trends{\^A}{\textregistered} in Machine Learning},
  volume={10},
  number={1-2},
  pages={1--141},
  year={2017},
  publisher={Emerald Publishing Limited}
}

@article{gretton2012optimal,
  title={Optimal kernel choice for large-scale two-sample tests},
  author={Gretton, Arthur and Sejdinovic, Dino and Strathmann, Heiko and Balakrishnan, Sivaraman and Pontil, Massimiliano and Fukumizu, Kenji and Sriperumbudur, Bharath K},
  journal={Advances in neural information processing systems},
  volume={25},
  year={2012}
}

@article{ribero2026regularized,
  title={Regularized $ f $-Divergence Kernel Tests},
  author={Ribero, M{\'o}nica and Schrab, Antonin and Gretton, Arthur},
  journal={arXiv preprint arXiv:2601.19755},
  year={2026}
}

@article{szekely2013energy,
  title={Energy statistics: A class of statistics based on distances},
  author={Sz{\'e}kely, G{\'a}bor J and Rizzo, Maria L},
  journal={Journal of statistical planning and inference},
  volume={143},
  number={8},
  pages={1249--1272},
  year={2013},
  publisher={Elsevier}
}

@article{friedman1979multivariate,
  title={Multivariate generalizations of the Wald-Wolfowitz and Smirnov two-sample tests},
  author={Friedman, Jerome H and Rafsky, Lawrence C},
  journal={The Annals of statistics},
  pages={697--717},
  year={1979},
  publisher={JSTOR}
}

@article{chen2017new,
  title={A new graph-based two-sample test for multivariate and object data},
  author={Chen, Hao and Friedman, Jerome H},
  journal={Journal of the American statistical association},
  volume={112},
  number={517},
  pages={397--409},
  year={2017},
  publisher={Taylor \& Francis}
}

@article{ghoshdastidar2018practical,
  title={Practical methods for graph two-sample testing},
  author={Ghoshdastidar, Debarghya and Von Luxburg, Ulrike},
  journal={Advances in Neural Information Processing Systems},
  volume={31},
  year={2018}
}

@article{wildberger2023flow,
  title={Flow matching for scalable simulation-based inference},
  author={Wildberger, Jonas and Dax, Maximilian and Buchholz, Simon and Green, Stephen and Macke, Jakob H and Sch{\"o}lkopf, Bernhard},
  journal={Advances in Neural Information Processing Systems},
  volume={36},
  pages={16837--16864},
  year={2023}
}

@article{hemerik2020robust,
  title={Robust testing in generalized linear models by sign flipping score contributions},
  author={Hemerik, Jesse and Goeman, Jelle J and Finos, Livio},
  journal={Journal of the Royal Statistical Society Series B: Statistical Methodology},
  volume={82},
  number={3},
  pages={841--864},
  year={2020},
  publisher={Oxford University Press}
}

@article{pooladian2023multisample,
  title={Multisample flow matching: Straightening flows with minibatch couplings},
  author={Pooladian, Aram-Alexandre and Ben-Hamu, Heli and Domingo-Enrich, Carles and Amos, Brandon and Lipman, Yaron and Chen, Ricky TQ},
  journal={arXiv preprint arXiv:2304.14772},
  year={2023}
}

@incollection{hoeffding1992class,
  title={A class of statistics with asymptotically normal distribution},
  author={Hoeffding, Wassily},
  booktitle={Breakthroughs in statistics: Foundations and basic theory},
  pages={308--334},
  year={1992},
  publisher={Springer}
}

@inproceedings{wu2018group,
  title={Group normalization},
  author={Wu, Yuxin and He, Kaiming},
  booktitle={Proceedings of the European conference on computer vision (ECCV)},
  pages={3--19},
  year={2018}
}

@article{kingma2014adam,
  title={Adam: A method for stochastic optimization},
  author={Kingma, Diederik P and Ba, Jimmy},
  journal={arXiv preprint arXiv:1412.6980},
  year={2014}
}
\bibliographystyle{iclr2027_conference}

\newpage

\appendix
\section*{Appendix}

\section{Proofs} \label{app:all_prfs}

\subsection{Proof of \Cref{thm:zflux}} \label{app:prf_zflux}
\begin{proof}
We first give a measure-theoretic formulation of the probability flux, which does not require the interpolating distributions to admit densities. Recall that
\[
    Z_t=(1-t)X+tY,
    \qquad
    p_t:=\Law(Z_t).
\]
Let $\mathcal B(\R^d)$ denote the Borel $\sigma$-algebra on $\R^d$. We define the vector-valued measure
\[
    \bJ_t(A)
    :=
    \E\!\left[
        (Y-X)\mathbf{1}_{\{Z_t\in A\}}
    \right],
    \qquad
    A\in\mathcal B(\R^d).
\]
Let $\|\bJ_t\|_{\mathrm{TV}}:=|\bJ_t|(\R^d)$ denote the total variation norm of $\bJ_t$. Since $p,q\in\mathcal P_1(\R^d)$,
\[
    \|\bJ_t\|_{\mathrm{TV}}
    \le
    \E\|Y-X\|
    \le
    \E\|X\|+\E\|Y\|
    <\infty,
\]
so $\bJ_t$ is a finite vector-valued measure. Moreover, by the tower property, for every $A\in\mathcal B(\R^d)$,
\begin{align*}
    \bJ_t(A)
    &=
    \E\!\left[
        (Y-X)\mathbf{1}_{\{Z_t\in A\}}
    \right] \\
    &=
    \E\!\left[
        \E[Y-X\mid Z_t]\,
        \mathbf{1}_{\{Z_t\in A\}}
    \right] \\
    &=
    \int_A
        \E[Y-X\mid Z_t=\bz]
        \,\dd p_t(\bz).
\end{align*}
Hence $\bJ_t\ll p_t$, with Radon--Nikodym derivative
\[
    \frac{\dd\bJ_t}{\dd p_t}(\bz)
    =
    \E[Y-X\mid Z_t=\bz]
    =
    \bv_t(\bz),
    \qquad
    p_t\text{-a.e.}
\]
Whenever $p_t$ admits a density with respect to the Lebesgue measure, which we also denote by $p_t(z)$, this reduces to
\[
    \dd\bJ_t(\bz)
    =
    p_t(\bz)\bv_t(\bz)\,\dd \bz
    =
    \bj_t(\bz)\,\dd \bz.
\]
Thus $\bJ_t$ is the measure-theoretic counterpart of the probability flux $\bj_t$ used in the main text.

We first prove the forward direction. Suppose $p=q$. Since $X$ and $Y$ are independent and identically distributed, $(X,Y)$ is exchangeable. Therefore, for every $A\in\mathcal B(\R^d)$,
\begin{align*}
    \bJ_{1-t}(A)
    &=
    \E\!\left[
        (Y-X)
        \mathbf{1}_{\{tX+(1-t)Y\in A\}}
    \right] \\
    &=
    \E\!\left[
        (X-Y)
        \mathbf{1}_{\{tY+(1-t)X\in A\}}
    \right] \\
    &=
    -\E\!\left[
        (Y-X)
        \mathbf{1}_{\{(1-t)X+tY\in A\}}
    \right] \\
    &=
    -\bJ_t(A).
\end{align*}
Hence
\[
    \bJ_t+\bJ_{1-t}=0,
    \qquad
    \forall t\in[0,1].
\]
When the fluxes admit densities, this is exactly
\[
    \bj_t+\bj_{1-t}=\bzero.
\]

Conversely, suppose
\[
    \bJ_t+\bJ_{1-t}=0,
    \qquad
    \forall t\in[0,1].
\]
For any test function $\varphi\in C_c^1(\R^d)$, where $C_c^1(\R^d)$ denotes the space of continuously differentiable functions with compact support, define
\[
    F_\varphi(t)
    :=
    \int_{\R^d}\varphi(\bz)\,\dd p_t(\bz)
    =
    \E[\varphi(Z_t)].
\]
Since $\nabla\varphi$ is bounded and
$\E\|Y-X\|<\infty$, differentiation under the expectation is justified by dominated convergence, and
\[
    F_\varphi'(t)
    =
    \E\!\left[
        \nabla\varphi(Z_t)^\top(Y-X)
    \right]
    =
    \int_{\R^d}
        \nabla\varphi(\bz)^\top\,\dd\bJ_t(\bz).
\]
This is the weak form of the continuity equation \citep{ambrosio2005gradient}.

Now consider the difference between the marginals at times $t$ and $1-t$. We have
\begin{align*}
    \frac{\dd}{\dd t}
    \int_{\R^d}
        \varphi(\bz)\,
        \dd\bigl(p_t-p_{1-t}\bigr)(\bz)
    &=
    F_\varphi'(t)+F_\varphi'(1-t) \\
    &=
    \int_{\R^d}
        \nabla\varphi(\bz)^\top
        \,\dd\bigl(\bJ_t+\bJ_{1-t}\bigr)(\bz) \\
    &=0.
\end{align*}
Therefore,
\[
    \int_{\R^d}
        \varphi(z)\,
        \dd\bigl(p_t-p_{1-t}\bigr)(z)
\]
is constant in $t$. At $t=1/2$, this quantity is zero, since
$p_{1/2}-p_{1/2}=0$. Hence,
\[
    \int_{\R^d}\varphi(\bz)\,\dd p_t(\bz)
    =
    \int_{\R^d}\varphi(\bz)\,\dd p_{1-t}(\bz),
    \qquad
    \forall \varphi\in C_c^1(\R^d),
    \quad
    \forall t\in[0,1].
\]
Since $C_c^1(\R^d)$ is measure determining, it follows that
\[
    p_t=p_{1-t},
    \qquad
    \forall t\in[0,1].
\]
Taking $t=0$ gives
\[
    p=p_0=p_1=q.
\]
This completes the proof.
\end{proof}

\subsection{Proof of \Cref{prop:zfd_w1}} \label{app:prf_zfd_w1}
\begin{proof}
We first give a measure-theoretic definition of the normalized flux residual in \Cref{eq:gt_fluxform}. Let $\bJ_t$ be the vector-valued flux measure defined in Appendix \ref{app:prf_zflux}, and define
\[
    \bR_t:=\bJ_t+\bJ_{1-t}.
\]
Since $\bJ_t\ll p_t$ and $\bJ_{1-t}\ll p_{1-t}$, we have
\[
    \bR_t\ll p_t+p_{1-t}.
\]
We therefore define
\[
    \bg(\cdot,t)
    :=
    \frac{\dd\bR_t}{\dd(p_t+p_{1-t})}
    =
    \frac{\dd(\bJ_t+\bJ_{1-t})}
         {\dd(p_t+p_{1-t})}.
\]
Equivalently, since
\[
    \widetilde p_t=\frac{p_t+p_{1-t}}{2},
\]
we have
\begin{equation}
    \dd\bR_t(\bz)
    =
    2\bg(\bz,t)\,\dd\widetilde p_t(\bz).
    \label{eq:gt_measure}
\end{equation}
Whenever $p_t$ and $p_{1-t}$ admit densities with respect to the Lebesgue measure, this definition reduces to
\[
    \bg(\bz,t)
    =
    \frac{\bj_t(\bz)+\bj_{1-t}(\bz)}
         {p_t(\bz)+p_{1-t}(\bz)},
\]
which is precisely \Cref{eq:gt_fluxform}.

If $\ZFD(p,q)=+\infty$, the result is immediate. Hence, suppose
$\ZFD(p,q)<\infty$. For any continuously differentiable
$1$-Lipschitz function $\varphi:\R^d\to\R$, we have
\[
    \int_{\R^d}\varphi(\bz)\,\dd p_t(\bz)
    =
    \E[\varphi(Z_t)].
\]
Since $\|\nabla\varphi(\bz)\|\leq 1$ and
$\E\|Y-X\|<\infty$, the difference quotients are dominated by the
integrable random variable $\|Y-X\|$. Hence, differentiation under the
expectation is justified by dominated convergence, giving
\[
    \frac{\dd}{\dd t}
    \int_{\R^d}\varphi(\bz)\,\dd p_t(\bz)
    =
    \E\!\left[
        \nabla\varphi(Z_t)^\top(Y-X)
    \right]
    =
    \int_{\R^d}
        \nabla\varphi(\bz)^\top\,\dd\bJ_t(\bz).
\]
Applying this identity to $p_t-p_{1-t}$ and integrating over
$t\in[0,1/2]$ yields
\[
    \int_{\R^d}\varphi(\bz)\,\dd(p-q)(\bz)
    =
    -\int_0^{1/2}
        \int_{\R^d}
        \nabla\varphi(\bz)^\top
        \,\dd\bR_t(\bz)\,\dd t.
\]
Using \Cref{eq:gt_measure} and
$\|\nabla\varphi(\bz)\|\leq 1$, we obtain
\begin{align}
    \left|
    \int_{\R^d}\varphi(\bz)\,\dd(p-q)(\bz)
    \right|
    &\leq
    2\int_0^{1/2}
        \int_{\R^d}
        \|\bg(\bz,t)\|
        \,\dd\widetilde p_t(\bz)\,\dd t
    \label{eq:zfd_lip_bound}
    \\
    &=
    \E_{t\sim\operatorname{Unif}[0,\frac12]}
    \E_{\bz\sim\widetilde p_t}
    \bigl[\|\bg(\bz,t)\|\bigr].
    \notag
\end{align}
By the Cauchy-Schwarz inequality,
\[
    \E_{t\sim\operatorname{Unif}[0,\frac12]}
    \E_{\bz\sim\widetilde p_t}
    \bigl[\|\bg(\bz,t)\|\bigr]
    \leq
    \left(
    \E_{t\sim\operatorname{Unif}[0,\frac12]}
    \E_{\bz\sim\widetilde p_t}
    \bigl[\|\bg(\bz,t)\|^2\bigr]
    \right)^{1/2}
    =
    \sqrt{\ZFD(p,q)}.
\]
Combining the above inequalities gives
\begin{equation}
    \left|
    \int_{\R^d}\varphi(\bz)\,\dd(p-q)(\bz)
    \right|
    \leq
    \sqrt{\ZFD(p,q)}
    \label{eq:zfd_lip_final}
\end{equation}
for every continuously differentiable $1$-Lipschitz function $\varphi$.
By a standard smooth approximation argument, \Cref{eq:zfd_lip_final}
extends to every $1$-Lipschitz function.

Recall that the $1$-Wasserstein distance is defined by
\[
    W_1(p,q)
    :=
    \inf_{\pi\in\Pi(p,q)}
    \int_{\R^d\times\R^d}
        \|\bx-\by\|
        \,\dd\pi(\bx,\by),
\]
where $\Pi(p,q)$ denotes the set of couplings of $p$ and $q$. By the
Kantorovich-Rubinstein duality \citep[Ch.~5]{villani2009optimal},
\[
    W_1(p,q)
    =
    \sup_{\|\varphi\|_{\mathrm{Lip}}\leq 1}
    \left|
        \int_{\R^d}
            \varphi(\bz)\,\dd(p-q)(\bz)
    \right|.
\]
Since \Cref{eq:zfd_lip_final} holds for every $1$-Lipschitz $\varphi$,
\[
    W_1^2(p,q)
    =
    \left(
        \sup_{\|\varphi\|_{\mathrm{Lip}}\leq 1}
        \left|
            \int_{\R^d}
                \varphi(\bz)\,\dd(p-q)(\bz)
        \right|
    \right)^2
    \leq
    \ZFD(p,q),
\]
which proves the claim.
\end{proof}

\subsection{Proof of \Cref{coro:def_zfd}} \label{app:prf_def_zfd}
\begin{proof}
We first show symmetry. Let $\bJ_t$ and
\[
    \bR_t:=\bJ_t+\bJ_{1-t}
\]
be the flux measure and residual flux associated with $(p,q)$, as defined at the beginning of Appendix \ref{app:prf_zflux} and \ref{app:prf_zfd_w1}, respectively. For the reversed pair $(q,p)$, denote the corresponding quantities by $\bJ_t^{\mathrm{rev}}$ and $\bR_t^{\mathrm{rev}}$.

Swapping the roles of $X$ and $Y$ maps
\[
    Z_t=(1-t)X+tY
    \quad\text{to}\quad
    Z_{1-t}=(1-t)Y+tX
\]
and reverses the displacement $Y-X$. Hence,
\[
    \bJ_t^{\mathrm{rev}}
    =
    -\bJ_{1-t},
\]
and therefore
\[
    \bR_t^{\mathrm{rev}}
    =
    \bJ_t^{\mathrm{rev}}+\bJ_{1-t}^{\mathrm{rev}}
    =
    -\bR_t.
\]
Meanwhile, the symmetrized marginal is unchanged under the swap,
\[
    \widetilde p_t^{\mathrm{rev}}
    =
    \frac{p_{1-t}+p_t}{2}
    =
    \widetilde p_t.
\]
It follows from the definition of the normalized residual that
\[
    \bg^{\mathrm{rev}}(\bz,t)
    =
    -\bg(\bz,t),
    \qquad
    \widetilde p_t\text{-a.e.}
\]
Hence, by \Cref{def:zfd}
\[
    \ZFD(q,p)
    =
    \ZFD(p,q).
\]

Nonnegativity follows immediately from \Cref{def:zfd}
\[
    \ZFD(p,q)
    =
    \E_{t \sim \operatorname{Unif}[0,1/2]}
    \E_{\bz\sim\widetilde p_t}
    \bigl[\|\bg(\bz,t)\|^2\bigr]
    \geq 0.
\]

It remains to establish definiteness. If $p=q$, then by
\Cref{thm:zflux},
\[
    \bJ_t+\bJ_{1-t}=0,
    \qquad
    \forall t\in[0,1],
\]
and hence $\bg(\bz,t)=0$ almost everywhere, which gives
\[
    \ZFD(p,q)=0.
\]
Conversely, if $\ZFD(p,q)=0$, then \Cref{prop:zfd_w1} implies
\[
    W_1^2(p,q)
    \leq
    \ZFD(p,q)
    =
    0.
\]
Since $W_1$ is a metric on $\mathcal P_1(\R^d)$, we obtain $p=q$.
Therefore,
\[
    \ZFD(p,q)=0
    \quad\Longleftrightarrow\quad
    p=q.
\]
\end{proof}

\subsection{Proof of \Cref{lem:gt_regform}} \label{app:prf_gt_regform}

\begin{proof}
We first show that $\widetilde Z_t$ follows the symmetrized marginal.
Let $\mathcal B(\R^d)$ denote the Borel $\sigma$-algebra on $\R^d$.
For any
$A\in\mathcal B(\R^d)$, since $B\sim\operatorname{Ber}(1/2)$ is independent of
$(X,Y)$,
\begin{align*}
    \mathbb P(\widetilde Z_t\in A)
    &=
    \frac{1}{2}\mathbb P(Z_t\in A)
    +
    \frac{1}{2}\mathbb P(Z_{1-t}\in A) \\
    &=
    \frac{1}{2}p_t(A)
    +
    \frac{1}{2}p_{1-t}(A)
    =
    \widetilde p_t(A).
\end{align*}
Hence,
\[
    \Law(\widetilde Z_t)=\widetilde p_t.
\]

Next, recall from Appendix \ref{app:prf_zfd_w1} that
\[
    \bR_t:=\bJ_t+\bJ_{1-t},
    \qquad
    \dd\bR_t(\bz)
    =
    2\bg(\bz,t)\,\dd\widetilde p_t(\bz).
\]
For any $A\in\mathcal B(\R^d)$,
\begin{align*}
    \E\!\left[
        (Y-X)\mathbf 1_{\{\widetilde Z_t\in A\}}
    \right]
    &=
    \frac{1}{2}
    \E\!\left[
        (Y-X)\mathbf 1_{\{Z_t\in A\}}
    \right]
    +
    \frac{1}{2}
    \E\!\left[
        (Y-X)\mathbf 1_{\{Z_{1-t}\in A\}}
    \right] \\
    &=
    \frac{1}{2}\bR_t(A) \\
    &=
    \int_A
        \bg(\bz,t)\,\dd\widetilde p_t(\bz).
\end{align*}
Since $\Law(\widetilde Z_t)=\widetilde p_t$, this identifies $\bg(\cdot,t)$
as the conditional expectation,
\[
    \bg(\bz,t)
    =
    \E[Y-X\mid \widetilde Z_t=\bz],
    \qquad
    \widetilde p_t\text{-a.e.}
\]

For the regression characterization, we additionally assume $p,q\in\mathcal P_2(\R^d)$. Let $t\sim\operatorname{Unif}[0,1/2]$ be independent of $(X,Y,B)$. Then $\E\|Y-X\|^2<\infty$, and by the standard $L^2$ projection property of conditional expectation, for any measurable $\bu$ with finite second moment,
\begin{align*}
    &\E\left[
        \|Y-X-\bu(\widetilde Z_t,t)\|^2
    \right] \\
    &\qquad=
    \E\left[
        \|Y-X-\bg(\widetilde Z_t,t)\|^2
    \right]
    +
    \E\left[
        \|\bu(\widetilde Z_t,t)-\bg(\widetilde Z_t,t)\|^2
    \right].
\end{align*}
Therefore, the unique minimizer up to almost-everywhere equivalence is
\[
    \bg(\cdot,\cdot)
    =
    \arg\min_{\bu}
    \E_{t\sim\operatorname{Unif}[0,1/2],\,X,Y,B}
    \left[
        \|Y-X-\bu(\widetilde Z_t,t)\|^2
    \right].
\]
This proves \Cref{eq:gt_regform}.
\end{proof}

\subsection{Tower Property for the Alignment}
\begin{lemma}[Tower property]
\label{lem:tower_property}
Let $t\sim\operatorname{Unif}[0,1/2]$ and
$\mu=\Law(\widetilde Z_t,t)$. For any
$\bh\in L^2(\mu;\R^d)$ such that $\E[\|\bh(\widetilde Z_t,t)\|\,\|Y-X\|]<\infty$, we have
\[
    \langle \bh,\bg\rangle_\mu
    =
    \E\!\left[
        \left\langle \bh(\widetilde Z_t,t),Y-X 
    \right \rangle \right].
\]
\end{lemma}

\begin{proof}
By \Cref{lem:gt_regform}, we have $\bg(\bz,t)= \E[Y-X | \widetilde Z_t=\bz]$, hence
\[
    \bg(\widetilde Z_t,t)
    =
    \E[Y-X\mid \widetilde Z_t, t].
\]
Therefore, by the tower property,
\begin{align*}
    \langle\bh,\bg\rangle_\mu
    &=
    \E\!\left[
        \left\langle
            \bh(\widetilde Z_t,t),
            \bg(\widetilde Z_t,t)
        \right\rangle
    \right] \\
    &=
    \E\!\left[
        \left\langle
            \bh(\widetilde Z_t,t),
            \E[Y-X\mid \widetilde Z_t,t]
        \right\rangle
    \right] \\
    &=
    \E\!\left[
        \left\langle
            \bh(\widetilde Z_t,t),
            Y-X
        \right\rangle
    \right].
\end{align*}
\end{proof}

\subsection{Proof of \Cref{prop:zfd_dual}} \label{app:prf_zfd_dual}
\begin{proof}
Since $\mu=\Law(\widetilde Z_t,t)$ with
$t\sim\operatorname{Unif}[0,1/2]$, the definition of ZFD gives
\[
    \ZFD(p,q)
    =
    \E_\mu\!\left[\|\bg(\widetilde Z_t,t)\|^2\right]
    =
    \|\bg\|_\mu^2.
\]
Hence, the assumption $\ZFD(p,q)<\infty$ implies
$\bg\in L^2(\mu;\R^d)$.

For any $\bh\in L^2(\mu;\R^d)$ with $\|\bh\|_\mu>0$, the
Cauchy-Schwarz inequality yields
\[
    \langle \bh,\bg\rangle_\mu^2
    \leq
    \|\bh\|_\mu^2\|\bg\|_\mu^2.
\]
Therefore,
\[
    \frac{\langle \bh,\bg\rangle_\mu^2}
         {\|\bh\|_\mu^2}
    \leq
    \|\bg\|_\mu^2
    =
    \ZFD(p,q),
\]
and hence
\[
    \sup_{\substack{
        \bh\in L^2(\mu;\R^d)\\
        \|\bh\|_\mu>0
    }}
    \frac{\langle \bh,\bg\rangle_\mu^2}
         {\|\bh\|_\mu^2}
    \leq
    \ZFD(p,q).
\]

If $\ZFD(p,q)>0$, then $\|\bg\|_\mu>0$. Taking
$\bh=c\bg$ for any $c\neq0$ gives
\[
    \frac{\langle c\bg,\bg\rangle_\mu^2}
         {\|c\bg\|_\mu^2}
    =
    \frac{c^2\|\bg\|_\mu^4}
         {c^2\|\bg\|_\mu^2}
    =
    \|\bg\|_\mu^2
    =
    \ZFD(p,q).
\]
Thus the upper bound is attained by any nonzero rescaling of $\bg$, and
\[
    \ZFD(p,q)
    =
    \sup_{\substack{
        \bh\in L^2(\mu;\R^d)\\
        \|\bh\|_\mu>0
    }}
    \frac{\langle \bh,\bg\rangle_\mu^2}
         {\|\bh\|_\mu^2}.
\]

If $\ZFD(p,q)=0$, then $\|\bg\|_\mu=0$, so
$\bg=0$ $\mu$-almost everywhere. Consequently,
\[
    \langle \bh,\bg\rangle_\mu=0
\]
for every admissible $\bh$, and the variational objective is identically
zero. This proves the claim.
\end{proof}

\subsection{Proof of \Cref{prop:zf2st_asymptotic}} \label{app:prf_zf2st_asymptotic}

\begin{proof}
For brevity, write
\[
    \fraks_i
    :=
    \fraks_{\bh}(X_i,Y_i,t_i,B_i),
    \qquad
    i=1,\ldots,n.
\]
Since $\bh$ is fixed and
$\{(X_i,Y_i,t_i,B_i)\}_{i=1}^n$ are i.i.d.,
the random variables $\{\fraks_i\}_{i=1}^n$ are also i.i.d. By
\Cref{lem:tower_property},
\[
    \E[\fraks_i]
    =
    \tau(\bh),
    \qquad
    \var(\fraks_i)
    =
    \sigma^2(\bh).
\]
Since $0<\sigma^2(\bh)<\infty$, the classical central limit theorem gives
\[
    \frac{
        \sqrt{n}\bigl(\hat{\tau}(\bh)-\tau(\bh)\bigr)
    }{
        \sigma(\bh)
    }
    \xrightarrow{d}
    \mathcal N(0,1).
\]

Moreover, finite variance implies $\E[\fraks_i^2]<\infty$. Hence, by the
law of large numbers, the sample variance is consistent:
\[
    \hat{\sigma}^2(\bh)
    \xrightarrow{p}
    \sigma^2(\bh).
\]
Since $\sigma^2(\bh)>0$,
\[
    \frac{\hat{\sigma}(\bh)}{\sigma(\bh)}
    \xrightarrow{p}
    1.
\]
Applying Slutsky's theorem therefore yields
\[
    \frac{
        \sqrt{n}\bigl(\hat{\tau}(\bh)-\tau(\bh)\bigr)
    }{
        \hat{\sigma}(\bh)
    }
    =
    \frac{
        \sqrt{n}\bigl(\hat{\tau}(\bh)-\tau(\bh)\bigr)
    }{
        \sigma(\bh)
    }
    \frac{\sigma(\bh)}{\hat{\sigma}(\bh)}
    \xrightarrow{d}
    \mathcal N(0,1),
\]
which proves the claim.
\end{proof}

\subsection{Proof of \Cref{prop:power_approx}} \label{app:prf_power_approx}

\begin{proof}
Recall that $\hat\tau_{\rm te}(\bh)$ and
$\hat\sigma_{\rm te}(\bh)$ are the empirical mean and standard deviation
computed from the held-out test split $\mathcal S_{\rm te}$. We decompose
\[
    \zeta_{\rm ZF}
    =
    \frac{
        \sqrt{n_{\rm te}}
        \bigl(\hat\tau_{\rm te}(\bh)-\tau(\bh)\bigr)
    }{
        \hat\sigma_{\rm te}(\bh)
    }
    +
    \sqrt{n_{\rm te}}\,
    \frac{\tau(\bh)}
         {\hat\sigma_{\rm te}(\bh)}.
\]
By \Cref{prop:zf2st_asymptotic},
\[
    \frac{
        \sqrt{n_{\rm te}}
        \bigl(\hat\tau_{\rm te}(\bh)-\tau(\bh)\bigr)
    }{
        \hat\sigma_{\rm te}(\bh)
    }
    \xrightarrow{d}
    \mathcal N(0,1).
\]
Following the power analysis in \citet[Eq.~(4)]{kubler2022automl}, we
additionally assume that the population standard deviation $\sigma(\bh)$
is known for the Gaussian power approximation. Under this approximation,
\[
    \zeta_{\rm ZF}
    \approx
    Z
    +
    \sqrt{n_{\rm te}}\,
    \frac{\tau(\bh)}{\sigma(\bh)}
    =
    Z
    +
    \sqrt{n_{\rm te}}\,
    {\rm SNR}(\bh),
    \qquad
    Z\sim\mathcal N(0,1).
\]
Therefore,
\[
    \Pr\left(
        \zeta_{\rm ZF}
        \geq
        \Phi^{-1}(1-\alpha)
    \right)
    \approx
    1-\Phi\left(
        \Phi^{-1}(1-\alpha)
        -
        \sqrt{n_{\rm te}}\,
        {\rm SNR}(\bh)
    \right),
\]
which gives \Cref{eq:power}.

We next establish consistency directly, without relying on the Gaussian
power approximation. Under any fixed alternative satisfying
$\tau(\bh)>0$, the law of large numbers gives
\[
    \hat\tau_{\rm te}(\bh)
    \xrightarrow{p}
    \tau(\bh),
    \qquad
    \hat\sigma_{\rm te}(\bh)
    \xrightarrow{p}
    \sigma(\bh).
\]
Hence,
\[
    \frac{\zeta_{\rm ZF}}{\sqrt{n_{\rm te}}}
    =
    \frac{\hat\tau_{\rm te}(\bh)}
         {\hat\sigma_{\rm te}(\bh)}
    \xrightarrow{p}
    \frac{\tau(\bh)}{\sigma(\bh)}
    =
    {\rm SNR}(\bh)
    >0.
\]
It follows that
\[
    \zeta_{\rm ZF}
    \xrightarrow{p}
    +\infty,
\]
and therefore
\[
    \Pr\left(
        \zeta_{\rm ZF}
        \geq
        \Phi^{-1}(1-\alpha)
    \right)
    \longrightarrow 1.
\]
\end{proof}

\subsection{Finite-sample validity and consistency of the wild bootstrap} \label{app:prf_finite_validity_consistency}
We here establish the finite-sample validity and consistency of the wild-bootstrap calibration used in \textbf{Stage-II} of \Cref{sec:practical_implementation}.

\begin{proposition}[Validity and consistency of wild-bootstrap calibration]
\label{prop:signflip_validity}
Suppose the witness $\bh$ is fixed and independent of the held-out test split $\mathcal S_{\rm te}$.
Under $H_0:p=q$, the exhaustive wild-bootstrap calibration based on $\hat\tau_{\rm te}(\bh)$ controls the finite-sample Type-I error at level $\alpha$.
The same guarantee holds for the corrected Monte Carlo $p$-value $\hat p_{\rm val}$ in \Cref{eq:wild_bootstrap}.

Under any fixed alternative satisfying the positive alignment $\tau(\bh)>0$ with $\sigma^2(\bh)<\infty$, the exhaustive wild-bootstrap test is consistent, i.e., its power converges
to one as $n_{\rm te}\to\infty$. The corrected Monte Carlo test is also consistent provided $\tfrac{1}{n_{\rm perm}+1}\leq\alpha$.
\end{proposition}

\begin{proof}
Condition on the fixed witness $\bh$, which is independent of the
held-out test split $\mathcal S_{\rm te}$. For brevity, write
\[
    \fraks_i
    :=
    \fraks_{\bh}(X_i,Y_i,t_i,B_i),
    \qquad
    i=1,\ldots,n_{\rm te}.
\]

We first establish finite-sample validity under $H_0$. For
$\epsilon\in\{-1,+1\}$, define the pairwise transformation
\[
    \operatorname{swap}_{\epsilon}(X,Y,t,B)
    :=
    \begin{cases}
        (X,Y,t,B), & \epsilon=+1,\\
        (Y,X,t,1-B), & \epsilon=-1.
    \end{cases}
\]
For a sign vector
\[
    \boldsymbol{\epsilon}
    =
    (\epsilon_1,\ldots,\epsilon_{n_{\rm te}})
    \in\{-1,+1\}^{n_{\rm te}},
\]
let $\operatorname{swap}_{\boldsymbol{\epsilon}}$ apply
$\operatorname{swap}_{\epsilon_i}$ independently to the $i$-th held-out tuple
$(X_i,Y_i,t_i,B_i)$.

Under $H_0:p=q$, $X_i$ and $Y_i$ are i.i.d., and
$B_i\sim\operatorname{Ber}(1/2)$. Hence, for every $i$,
\[
    (X_i,Y_i,t_i,B_i)
    \overset{d}{=}
    (Y_i,X_i,t_i,1-B_i).
\]
Since the held-out tuples are independent, for every
$\boldsymbol{\epsilon}\in\{-1,+1\}^{n_{\rm te}}$,
\[
    \operatorname{swap}_{\boldsymbol{\epsilon}}
    \left(
        \{(X_i,Y_i,t_i,B_i)\}_{i=1}^{n_{\rm te}}
    \right)
    \overset{d}{=}
    \{(X_i,Y_i,t_i,B_i)\}_{i=1}^{n_{\rm te}}.
\]

For each transformed tuple, the randomized interpolation state is unchanged
while the displacement is reversed, and therefore
\[
    \fraks_{\bh}(Y_i,X_i,t_i,1-B_i)
    =
    -\fraks_{\bh}(X_i,Y_i,t_i,B_i).
\]
Consequently,
\[
    \operatorname{swap}_{\boldsymbol{\epsilon}}
    :
    (\fraks_1,\ldots,\fraks_{n_{\rm te}})
    \longmapsto
    (\epsilon_1\fraks_1,\ldots,
     \epsilon_{n_{\rm te}}\fraks_{n_{\rm te}}).
\]
Thus, applying $\operatorname{swap}_{\boldsymbol{\epsilon}}$ to the held-out
data is exactly equivalent to multiplying the pre-computed score vector by
Rademacher signs, which is precisely the wild-bootstrap replicate in
\Cref{eq:wild_bootstrap}. Since the transformed held-out sample has the same
distribution as the original sample under $H_0$, the exhaustive
wild-bootstrap calibration is an exact finite randomization test and controls
the Type-I error at level $\alpha$.

When only $n_{\rm perm}$ Rademacher sign vectors are sampled, the corrected
Monte Carlo $p$-value in \Cref{eq:wild_bootstrap} preserves finite-sample
level control by the standard randomization-test argument
\citep{hemerik2018exact}.

We next establish consistency under a fixed alternative. Recall that $\E[\fraks_i]=\tau(\bh)>0$.
Hence, by the law of large numbers,
\[
    \hat\tau_{\rm te}(\bh)
    =
    \frac{1}{n_{\rm te}}
    \sum_{i=1}^{n_{\rm te}}\fraks_i
    \xrightarrow{p}
    \tau(\bh)>0.
\]

For a wild-bootstrap replicate, let
$\epsilon_1,\ldots,\epsilon_{n_{\rm te}}$ be i.i.d. Rademacher variables and
define
\[
    \hat\tau_{\rm te}^{*}(\bh)
    :=
    \frac{1}{n_{\rm te}}
    \sum_{i=1}^{n_{\rm te}}
        \epsilon_i\fraks_i.
\]
Conditionally on the held-out scores, we have
\[
    \E_{\epsilon}
    \bigl[
        \hat\tau_{\rm te}^{*}(\bh)
    \bigr]
    =0,
    \quad
    \var_{\epsilon}
    \bigl[
        \hat\tau_{\rm te}^{*}(\bh)
    \bigr]
    =
    \frac{1}{n_{\rm te}^2}
    \sum_{i=1}^{n_{\rm te}}\fraks_i^2.
\]
Since $\E[\fraks_i^2]= \sigma^2(\bh)+\tau^2(\bh)<\infty$,
the law of large numbers gives
\[
    \frac{1}{n_{\rm te}}
    \sum_{i=1}^{n_{\rm te}}\fraks_i^2
    =O_p(1),
\]
and hence
\[
    \var_{\epsilon}
    \bigl[
        \hat\tau_{\rm te}^{*}(\bh)
    \bigr]
    =
    O_p(n_{\rm te}^{-1}).
\]
Therefore, the wild-bootstrap distribution concentrates around zero, whereas
the observed statistic $\hat\tau_{\rm te}(\bh)$ converges to the strictly
positive population alignment $\tau(\bh)$. It follows that the power of the
exhaustive wild-bootstrap test converges to one.

Finally, consider the corrected Monte Carlo calibration with a fixed
$n_{\rm perm}$. Each sampled wild-bootstrap replicate also concentrates
around zero, while
$\hat\tau_{\rm te}(\bh)\xrightarrow{p}\tau(\bh)>0$. Consequently,
\[
    \hat p_{\rm val}
    \xrightarrow{p}
    \frac{1}{n_{\rm perm}+1}.
\]
Therefore, whenever
\[
    \frac{1}{n_{\rm perm}+1}\leq\alpha,
\]
the corrected Monte Carlo test is also consistent.
\end{proof}

\begin{figure}[t]
    \centering
    \includegraphics[width=0.8\linewidth]{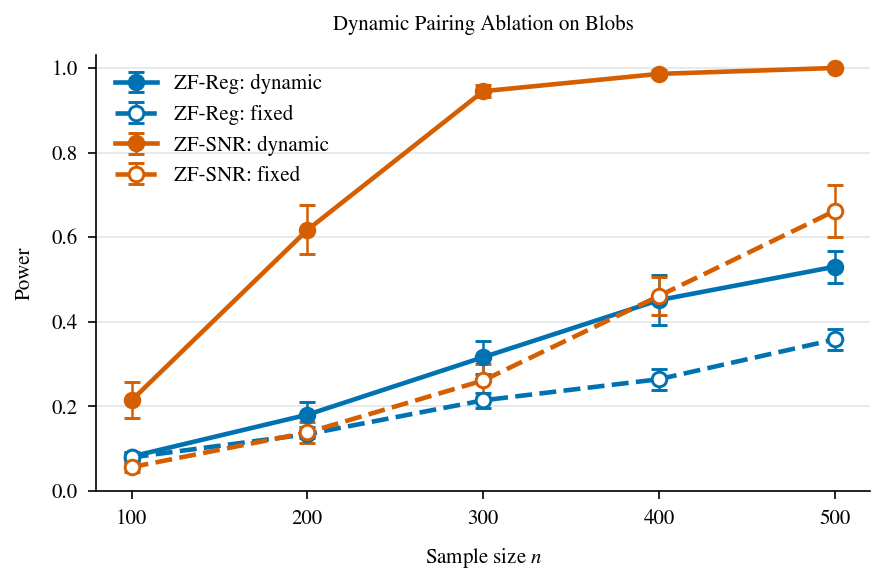}
    \caption{
    Dynamic pairing ablation for both tests on the \texttt{Blobs} dataset.
    }
    \label{fig:dyn_pair}
\end{figure}

\section{Implementation Details in ZF2ST} \label{app:implementation_details}

\subsection{Full Algorithm}

We summarize the full ZF2ST procedure in \Cref{alg:zf2st}.

\begin{algorithm}[H]
\caption{Zero-Flow Two-Sample Test (ZF2ST).}
\label{alg:zf2st}
\footnotesize
\begin{minipage}[t]{0.48\linewidth}

\textbf{Stage I: ZF-Reg}
\begin{algorithmic}[1]
    \State \textbf{Input:} $\calS_p$, $\calS_q$, level $\alpha$
    \State  Split: $\calS_p=\calS_p^{\rm tr}\cup\calS_p^{\rm te}$, $\calS_q=\calS_q^{\rm tr}\cup\calS_q^{\rm te}$.
    \State Initialize witness $\bh_\theta$
    \For{$r=1,\ldots,N_{\rm iter}$}
    \State $\{(X_i,Y_i)\}_{i=1}^m
    \gets
    \Call{DP}{\calS_p^{\rm tr},\calS_q^{\rm tr},m}$
    \Comment{{\scriptsize Alg. \ref{alg:dynamic_pairing}}}

    \State Sample
    $t_{1:m}\overset{\rm i.i.d.}{\sim}\operatorname{Unif}[0,1/2]$

    \State Sample
    $B_{1:m}\overset{\rm i.i.d.}{\sim}\operatorname{Ber}(1/2)$
    
    \State Form 
    $\{(X_i,Y_i,t_i,B_i)\}_{i=1}^m$

    \State Compute
    $\widehat{\calR}_{\rm reg}(\bh_\theta)$ using \eqref{eq:learning_reg}

    \State Update $\bh_\theta$:
    $\theta
    \gets
    \theta-\eta\nabla_{\rm Adam}
    \widehat{\calR}_{\rm reg}(\bh_\theta)$
    \EndFor
    \State $\bh \gets \bh_\theta$ \Comment{Fixed witness}
\end{algorithmic}
\end{minipage}
\hfill
\begin{minipage}[t]{0.48\linewidth}
\textbf{Stage I: ZF-SNR}
\begin{algorithmic}[1]
    \State \textbf{Input:} $\calS_p$, $\calS_q$, level $\alpha$
    \State Split: $\calS_p=\calS_p^{\rm tr}\cup\calS_p^{\rm te}$, 
    $\calS_q=\calS_q^{\rm tr}\cup\calS_q^{\rm te}$.
    \State Initialize witness $\bh_\theta$
    \For{$r=1,\ldots,N_{\rm iter}$}
        \State $\{(X_i,Y_i)\}_{i=1}^m
        \gets
        \Call{DP}{\calS_p^{\rm tr},\calS_q^{\rm tr},m}$
        \Comment{{\scriptsize Alg. \ref{alg:dynamic_pairing}}}

        \State Sample
        $t_{1:m,1:K}\overset{\rm i.i.d.}{\sim}
        \operatorname{Unif}[0,1/2]$

        \State \# Branch $B$ is averaged out

        \State Form
        $\{(X_i,Y_i,t_{i,1:K})\}_{i=1}^m$

        \State Compute
        $\widehat{\calR}_{{\rm snr},K}(\bh_\theta)$
        using \eqref{eq:empirical_snr_vr}

        \State Update $\bh_\theta$:
        $\theta
        \gets
        \theta-\eta\nabla_{\rm Adam}
        \widehat{\calR}_{{\rm snr},K}(\bh_\theta)$
    \EndFor
    \State $\bh \gets \bh_\theta$ \Comment{Fixed witness}
\end{algorithmic}
\end{minipage}

\medskip
\hrule
\medskip

\textbf{Stage II: Wild bootstrap calibration}

\begin{algorithmic}[1]

    \State 
    $\{(X_i,Y_i)\}_{i=1}^{n_{\rm te}}
    \gets
    (\calS_p^{\rm te},\calS_q^{\rm te})$
    
    \State Sample
    $t_{1:n_{\rm te},1:K}
    \overset{\rm i.i.d.}{\sim}
    \operatorname{Unif}[0,1/2]$

    \State Form
    $\{(X_i,Y_i,t_{i,1:K})\}_{i=1}^{n_{\rm te}}$

    \State $\displaystyle
    \fraks_i
    \gets
    \fraks_{\bh,K}(X_i,Y_i,t_{i,1:K}),
    \quad i=1,\ldots,n_{\rm te}$

    \State 
    $\displaystyle
    \hat\tau_{\rm te}(\bh)
    \gets
    \frac{1}{n_{\rm te}}
    \sum_{i=1}^{n_{\rm te}}
    \fraks_i$

    \For{$b=1,\ldots,n_{\rm perm}$}
        \State Draw
        $\epsilon_{1:n_{\rm te}}^{(b)}
        \overset{\rm i.i.d.}{\sim}
        \{-1,+1\}$ with equal prob.
        
        \State 
        $\displaystyle
        \hat\tau_{\rm te}^{(b)}(\bh)
        \gets
        \frac{1}{n_{\rm te}}
        \sum_{i=1}^{n_{\rm te}}
        \epsilon_i^{(b)}\fraks_i$
    \EndFor

    \State 
    $\displaystyle
    \hat p_{\rm val}
    \gets
    \frac{
        1+
        \sum_{b=1}^{n_{\rm perm}}
        \mathbf 1
        \left\{
            \hat\tau_{\rm te}^{(b)}(\bh)
            \geq
            \hat\tau_{\rm te}(\bh)
        \right\}
    }{
        n_{\rm perm}+1
    }$

    \State \Return $1$ \textbf{if} $\hat p_{\rm val}\leq\alpha$ \textbf{else} $0$
\end{algorithmic}
\end{algorithm}

\begin{figure}[h]
    \centering
    \begin{minipage}[t]{0.48\linewidth}
    \begin{algorithm}[H]
    \caption{Fixed pairing (FP).}
    \label{alg:fixed_pairing}
    \begin{lstlisting}[style=pythonstyle]
    def fixed_pairing(
        x, y, batch_size
    ):
        # shared indices for x and y
        idx = randint(
            len(x), batch_size
        )
    
        xb = x[idx]
        yb = y[idx]
    
        return xb, yb
    \end{lstlisting}
    \end{algorithm}
    \end{minipage}
    \hfill
    \begin{minipage}[t]{0.48\linewidth}
    \begin{algorithm}[H]
    \caption{Dynamic pairing (DP).}
    \label{alg:dynamic_pairing}
    \begin{lstlisting}[style=pythonstyle]
    def dynamic_pairing(
        x, y, batch_size
    ):
        # independently sample x and y
        ix = randint(
            len(x), batch_size
        )
        iy = randint(
            len(y), batch_size
        )
    
        xb = x[ix]
        yb = y[iy]
    
        return xb, yb
    \end{lstlisting}
    \end{algorithm}
    \end{minipage}
    \caption{Comparison between Fixed pairing and Dynamic Pairing}
    \label{fig:alg_DP}
\end{figure}

\subsection{Training Details: Dynamic Pairing} \label{app:detail_dyn_pair}

During witness training, we \textit{dynamically} construct cross-pairs by sampling the indices of the two training sets independently with replacement, as illustrated in \Cref{alg:dynamic_pairing}. The pairing of endpoint samples can be viewed as a choice of coupling, with standard flow-matching constructions typically relying on independent sampling from the two marginals \citep{pooladian2023multisample}. Let
\[
    \hat p_{\rm tr}
    :=
    \frac{1}{n_{\rm tr}}
    \sum_{i=1}^{n_{\rm tr}}\delta_{X_i},
    \qquad
    \hat q_{\rm tr}
    :=
    \frac{1}{n_{\rm tr}}
    \sum_{j=1}^{n_{\rm tr}}\delta_{Y_j}.
\]
Independent index sampling then draws each training pair directly from the empirical product coupling
\[
    \hat p_{\rm tr}\otimes\hat q_{\rm tr}
    =
    \frac{1}{n_{\rm tr}^2}
    \sum_{i=1}^{n_{\rm tr}}
    \sum_{j=1}^{n_{\rm tr}}
    \delta_{(X_i,Y_j)},
\]
which directly mirrors the independent population coupling $p\otimes q$ used in ZF2ST. Re-sampling the indices at every optimization iteration explores different cross-pair combinations without explicitly constructing all $n_{\rm tr}^2$ possible pairs.
Since the indices are sampled \textit{with replacement} from the empirical product coupling, the minibatch size is not restricted by $n_{\rm tr}$ and may exceed the number of training samples, allowing more cross-pair evaluations within each optimization step.

We further examine the effect of dynamic pairing on the \texttt{Blobs} dataset. 
As shown in \Cref{fig:dyn_pair}, dynamically resampling cross-pairs consistently improves the test power of both ZF-Reg and ZF-SNR compared with a fixed one-to-one pairing, with a more pronounced gain for ZF-SNR.

Importantly, dynamic pairing is used \textit{only} during training. At test time, the held-out pairs are kept fixed so that the resulting scores remain i.i.d., as assumed in \Cref{prop:zf2st_asymptotic}. Averaging over cross-pairs at test time would instead lead to a two-sample U-statistic, requiring a different asymptotic analysis \citep{hoeffding1992class,serfling2009approximation}.

\subsection{An Explicit Form of the Empirical Risk for the Variance-Reduced Objective} \label{app:explicit_emp_snr_vr}

We explicitly write out the empirical risk used for the variance-reduced ZF-SNR objective.
Recall that
\[
    \fraks_{\bh,K}(X,Y,t_{1:K})
    :=
    \frac{1}{2K}
    \sum_{k=1}^K
    \left\langle
        \bh(Z_{t_k},t_k)
        +
        \bh(Z_{1-t_k},t_k),
        Y-X
    \right\rangle .
\]
The empirical alignment is
\begin{equation}
\label{eq:empirical_vr_alignment}
    \hat\tau(\bh)
    :=
    \frac{1}{n}
    \sum_{i=1}^{n}
    \fraks_{\bh,K}(X_i,Y_i,t_{i,1:K}).
\end{equation}
The empirical variance is
\[
    \hat\sigma_K^2(\bh)
    :=
    \frac{1}{n-1}
    \sum_{i=1}^{n}
    \left(
        \fraks_{\bh,K}(X_i,Y_i,t_{i,1:K})
        -
        \hat\tau(\bh)
    \right)^2.
\]

Define the $K$-sample witness energy
\[
    \calE_{\bh,K}(X,Y,t_{1:K})
    :=
    \frac{1}{2K}
    \sum_{k=1}^K
    \left[
        \|\bh(Z_{t_k},t_k)\|^2
        +
        \|\bh(Z_{1-t_k},t_k)\|^2
    \right],
\]
where
\[
    Z_{t_k}=(1-t_k)X+t_kY,
    \qquad
    Z_{1-t_k}=t_kX+(1-t_k)Y.
\]

The regularized denominator is
\[
    \hat{\sigma}^{\lambda}_{K}(\bh)
    :=
    \left[
        \hat\sigma_K^2(\bh)
        +
        \frac{\lambda}{n}
        \sum_{i=1}^{n}
        \calE_{\bh,K}(X_i,Y_i,t_{i,1:K})
    \right]^{1/2}.
\]
Thus, on the training split $\calS^{\rm tr}$, the corresponding empirical risk is
\begin{equation}
\label{eq:empirical_snr_vr}
    \widehat{\calR}_{{\rm snr},K}(\bh)
    :=
    -\frac{
        \hat\tau_{\rm tr}(\bh)
    }{
        \hat{\sigma}^{\lambda}_{K,\rm tr}(\bh)
    }.
\end{equation}

\begin{table}[t]
\centering

\label{tab:complexity}
\begin{tabular}{lcc}
\toprule
Method
& Training per  update
& Testing \\
\midrule

ZF-Reg
& $\mathcal O(mC_{\bh})$
& $\mathcal O(2Kn_{\rm te}C_{\bh}+Bn_{\rm te})$ \\

ZF-SNR
& $\mathcal O(2KmC_{\bh})$
& $\mathcal O(2Kn_{\rm te}C_{\bh}+Bn_{\rm te})$ \\

MMD-D \citep{liu2020learning}
& $\mathcal O(mC_{\phi}+m^2C_k)$
& $\mathcal O(n_{\rm te}C_{\phi}+n_{\rm te}^2C_k+Bn_{\rm te}^2)$ \\

C2ST \citep{lopez2016revisiting}
& $\mathcal O(mC_f)$
& $\mathcal O(n_{\rm te}C_f+Bn_{\rm te})$ \\

MMD-FUSE \citep{biggs2023mmd}
& --
& $\mathcal O(cBn^2)$ \\

MMDagg \citep{schrab2023mmd}
& --
& $\mathcal O\!\left(c(B+B')n^2\right)$ \\
\bottomrule
\end{tabular}
\caption{
Leading-order computational complexity.
Here $m$ is the mini-batch size, $n_{\rm tr}$ and $n_{\rm te}$ are the
training and testing sample sizes with $n = n_{\rm tr} + n_{\rm te}$, $K$ is the number of time samples used
by ZF2ST, $B$ is the number of randomization/calibration replicates, and
$c$ is the number of candidate kernels. For MMDagg, $B'$ denotes the
additional randomization replicates used for aggregation-level correction.
We denote by $C_{\bh}$ the cost of one witness-network evaluation,
$C_{\phi}$ the cost of one feature-network evaluation,
$C_k$ the cost of one kernel evaluation, and $C_f$ the cost of one
classifier evaluation.
}
\end{table}

\subsection{Complexity of ZF2ST} \label{app:complexity}
Let $m$ denote the mini-batch size, $n_{\rm te}$ the number of held-out pairs, $K$ the number of time samples used in variance reduction, and $B$ the number of wild-bootstrap replicates. Let $C_{\bh}$ denote the cost of evaluating the witness network at one interpolation state; we suppress the fixed data dimension and constant factors.

\textbf{Training complexity.}
For ZF-Reg, each training pair requires one witness evaluation for the randomly sampled interpolation state, giving a per-update cost of $\mathcal O(mC_{\bh})$. For ZF-SNR, the variance-reduced objective evaluates both $Z_{t_k}$ and $Z_{1-t_k}$ for each of the $K$ time samples, resulting in a per-update cost of $\mathcal O(2KmC_{\bh})$. The empirical alignment, variance, and ridge terms require only linear-time reductions over these pre-computed witness outputs. Dynamic pairing independently samples the two sets of indices at every update and incurs an additional $\mathcal O(m)$ cost, without explicitly constructing the $n_{\rm tr}^2$ empirical cross-pairs.

\textbf{Testing complexity.}
Both ZF-Reg and ZF-SNR use the same variance-reduced score at test time. Pre-computing the held-out scores requires $\mathcal O(2Kn_{\rm te}C_{\bh})$. Once these scores have been computed, each wild-bootstrap replicate requires only Rademacher sign flips and a reduction over the $n_{\rm te}$ scores, with no additional witness evaluations. The total testing and calibration cost is therefore $\mathcal O(2Kn_{\rm te}C_{\bh}+Bn_{\rm te})$. In particular, the witness-evaluation cost does not scale with the number of wild-bootstrap replicates.

\section{Experimental Details} \label{app:experiment_detail}

\subsection{Protocols}

All tests are conducted at significance level $\alpha=0.05$.
For sample-splitting methods, we perform 10 independent training runs, each followed by 100 independent held-out test trials with the trained model fixed. Unless otherwise specified, we use $n=n_{\rm tr}+n_{\rm te}$ with $n_{\rm tr}=n_{\rm te}$; sample-splitting methods train on $n_{\rm tr}$ samples and evaluate on disjoint $n_{\rm te}$ held-out samples, whereas non-splitting methods use all $n$ samples in each of 100 independent trials. For CIFAR-10 vs.\ CIFAR-10.1, following \citet{liu2020learning}, we instead use $n_{\rm tr}=1000$ and $n_{\rm te}=1021$.
Standard errors for sample-splitting methods are computed across the 10 independent training runs, while those for non-splitting methods are computed across the 100 independent trials.
ZF2ST, MMD-D, C2ST, and KFDA-WiTS use 200 calibration randomizations, except AutoML, which uses its default 10,000 permutations.
MMD-FUSE and MMDagg use the default calibration settings of their official implementations.

KFDA-WiTS \citep{kubler2022witness} is evaluated only on non-image datasets, as neither the experiments in the paper nor the released code include an image-specific pipeline.

\subsection{Dataset Metadata and Preprocessing} \label{app:ds_preprocess}

\begin{table}[t]
\centering
\scriptsize
\caption{Summary of the real-world datasets used in our experiments.}
\label{tab:dataset_metadata}
\begin{tabularx}{\linewidth}{lrrlX}
\toprule
Dataset & Samples & Pixels & Data type & URL \\
\midrule
\texttt{Higgs}
& $11000000$
& --
& Tabular
& \url{https://archive.ics.uci.edu/dataset/280/higgs} \\

CIFAR-10 {\scriptsize (test)}
& $10000$
& $3\times32\times32$
& Image
& \url{https://www.cs.toronto.edu/~kriz/cifar.html} \\

CIFAR-10.1 {\scriptsize (v4\_data)}
& $2021$
& $3\times32\times32$
& Image
& \url{https://github.com/modestyachts/CIFAR-10.1/tree/master/datasets} \\

ImageNet {\scriptsize (validation)}
& $50000$
& Variable
& Image
& \url{https://www.image-net.org/challenges/LSVRC/2012/2012-downloads} \\

ImageNetV2 {\scriptsize (MatchedFrequency)}
& $10000$
& Variable
& Image 
& \url{https://huggingface.co/datasets/vaishaal/ImageNetV2/tree/main} \\
\bottomrule
\end{tabularx}
\end{table}

For the synthetic \texttt{Blobs} and \texttt{HDGM} datasets, we follow the distribution specifications given in \citet[Table~6]{liu2020learning}.

\Cref{tab:dataset_metadata} briefly summarizes the real-world datasets used in our experiments. 
We additionally provide the code repositories and pretrained checkpoints used to generate CIFAR-10 samples from the generative models considered in our experiments. 
\begin{itemize}
    \item \textbf{ScoreSDE} \citep{song2020score}: the code is available at \url{https://github.com/yang-song/score_sde_pytorch}. We use the 
    official pretrained checkpoint available at \url{https://drive.google.com/drive/folders/11s6A_xM7qiztdj8AHQWqaIAUSC3I7uX2}.

    \item \textbf{1-Rectified Flow} \citep{liu2022flow}: the code is available at \url{https://github.com/gnobitab/RectifiedFlow}.
    We use the official pretrained checkpoint available at  \url{https://drive.google.com/file/d/10aPF5KC30SjVwr6rOnNosStpSGXnELXn/view}.

    \item \textbf{MeanFlow} \citep{geng2026mean}: the code is available at \url{https://github.com/Gsunshine/py-meanflow}.
    We use the official pretrained checkpoint available at \url{https://drive.google.com/uc?id=1Ev9aZdt7hNVkOfaTQ7rrLAqs3jA_vnvj}.

    \item \textbf{StyleGAN-XL} \citep{sauer2022stylegan}: the code is available at \url{https://github.com/autonomousvision/stylegan-xl}.
    We use the official pretrained checkpoint available at \url{https://s3.eu-central-1.amazonaws.com/avg-projects/stylegan_xl/models/cifar10.pkl}.
\end{itemize}
We note that StyleGAN-XL is trained conditionally on CIFAR-10 class labels. Following the official generation procedure, we generate samples without specifying a class label, so the labels are sampled uniformly by the released code. For the Rectified Flow experiments that vary the sample size and number of function evaluations (NFE), we use the official Euler solver. For all remaining experiments, we keep the official sampling configurations unchanged.

Visualizations of all image datasets used in our experiments, including generated samples, are provided in Appendix~\ref{app:data_visualization}.

\textbf{Image preprocessing.}
Following \citet{liu2020learning}, all CIFAR-10, CIFAR-10.1, and generated CIFAR-10 images are kept in RGB format, resized from $32\times32$ to $64\times64$, and normalized channel-wise from $[0,1]$ to $[-1,1]$ using mean and standard deviation $(0.5,0.5,0.5)$, yielding $3\times64\times64$ inputs.

For ImageNet and ImageNetV2, we preserve the RGB channels, resize each image so that its shorter side is $256$ pixels, and apply a $224\times224$ center crop. The images are then normalized using the ImageNet channel-wise mean $(0.485,0.456,0.406)$ and standard deviation $(0.229,0.224,0.225)$, and finally resized to $3\times64\times64$.

\subsection{Configurations} \label{app:configurations}

All experiments are conducted on an NVIDIA RTX 5090 GPU, except KFDA-WiTS and AutoML, which are run on an AMD Ryzen 7 9700X 8-Core CPU.
We run MMD-D, C2ST-S, and C2ST-L using the official code provided by \citet{liu2020learning}.
For C2ST, we train a single shared classifier and compute both C2ST-S and C2ST-L statistics from its outputs.
For KFDA-WiTS, we use the official implementation \citep{kubler2022witness} but replace the FALKON-based solver with a full Gram-matrix implementation for compatibility.
AutoML is run using the official Python package \texttt{autotst} \citep{kubler2022automl}. We use the official Python packages \texttt{mmdagg} and \texttt{mmdfuse} released by \citet{schrab2023mmd} and \citet{biggs2023mmd}, respectively.

For the non-image experiments, we keep the original network configurations of MMD-D, C2ST-S, and C2ST-L from \citet{liu2020learning} unchanged whenever applicable. For both ZF-Reg and ZF-SNR, we \textit{match} the five-layer fully connected architecture used by MMD-D on \texttt{Blobs}, \texttt{HDGM}, and \texttt{Higgs}, using the same hidden widths of $50$, $3\times d$, and $20$, respectively, where $d$ denotes the data dimension. The only architectural changes are to concatenate the scalar time $t$ to the input, replace the Softplus activation with SiLU, and map the final hidden representation to the $d$-dimensional witness through a linear output layer.

For all image experiments, we preprocess the images to $3\times64\times64$ inputs (see Appendix~\ref{app:ds_preprocess}) and keep the original CIFAR-10 network configurations of MMD-D, C2ST-S, and C2ST-L from \citet[Figures~8 and~9]{liu2020learning} unchanged. Both ZF-Reg and ZF-SNR use the same convolutional witness architecture across all image experiments. As shown in \Cref{fig:cnn_arch}, the network consists of an initial convolutional layer, a residual block with Group Normalization (GN; \citealp{wu2018group}) and SiLU activations, and a final convolutional layer that outputs the witness.

For MMD-D, C2ST-S, C2ST-L, ZF-Reg, and ZF-SNR, we use the Adam optimizer \citep{kingma2014adam} for training in all experiments.

\begin{figure}[h]
    \centering
    \includegraphics[width=0.95\linewidth]{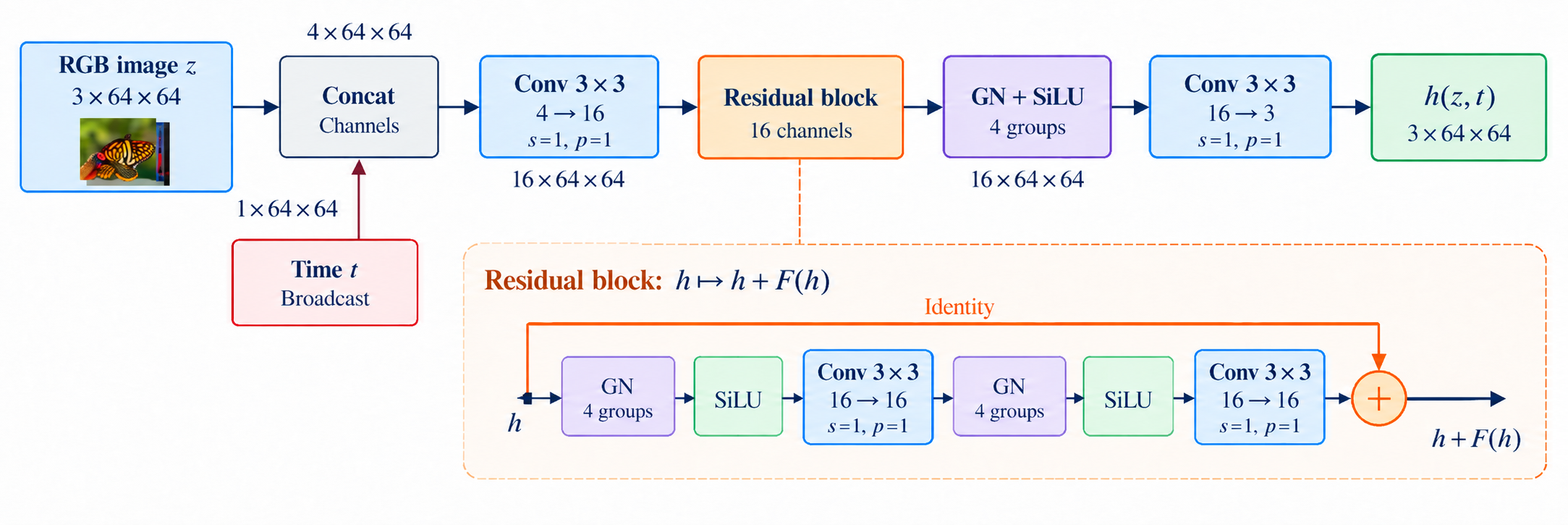}
    \caption{CNN of witness $\bh$ used in \textit{all} image experiments. All convolutional layers use a $3\times3$ kernel with stride $s=1$ and padding $p=1$.}
    \label{fig:cnn_arch}
\end{figure}

\subsection{Hyperparameters of test methods}

Following \citet{liu2020learning}, we use the default Adam hyperparameters except for the learning rate. Both ZF-Reg and ZF-SNR are trained for $2{,}000$ iterations in all experiments. For \texttt{Blobs}, \texttt{HDGM}, and \texttt{Higgs}, we use learning rates of $5\times10^{-3}$, $10^{-3}$, and $3\times10^{-3}$, respectively, with batch size $\min(2048,6\times n_{\rm tr})$. For all image experiments, we use a learning rate of $5\times10^{-4}$ and batch size $100$.
We use one validation set of the same size as the training set to select the above hyperparameters and $\lambda$.

For ZF-SNR, we use the variance-reduced objective throughout, with $(K,\lambda)=(3,10^{-3})$ for non-image experiments and $(2,10^{-2})$ for image experiments. The same values of $K$ are used for the variance-reduced test score of both ZF-Reg and ZF-SNR.

For MMD-D, C2ST-S, and C2ST-L, we keep all hyperparameters and their selection rules unchanged from the original implementations in \citet{liu2020learning}.

For AutoML, we set the training time parameter \texttt{time\_limit} to 60 seconds.

\subsection{Type-I Errors}
The corresponding Type-I error results are reported in \Cref{fig:type1} and \Cref{tab:type1_image}.

\begin{figure}[t]
    \centering
    \includegraphics[width=\linewidth]{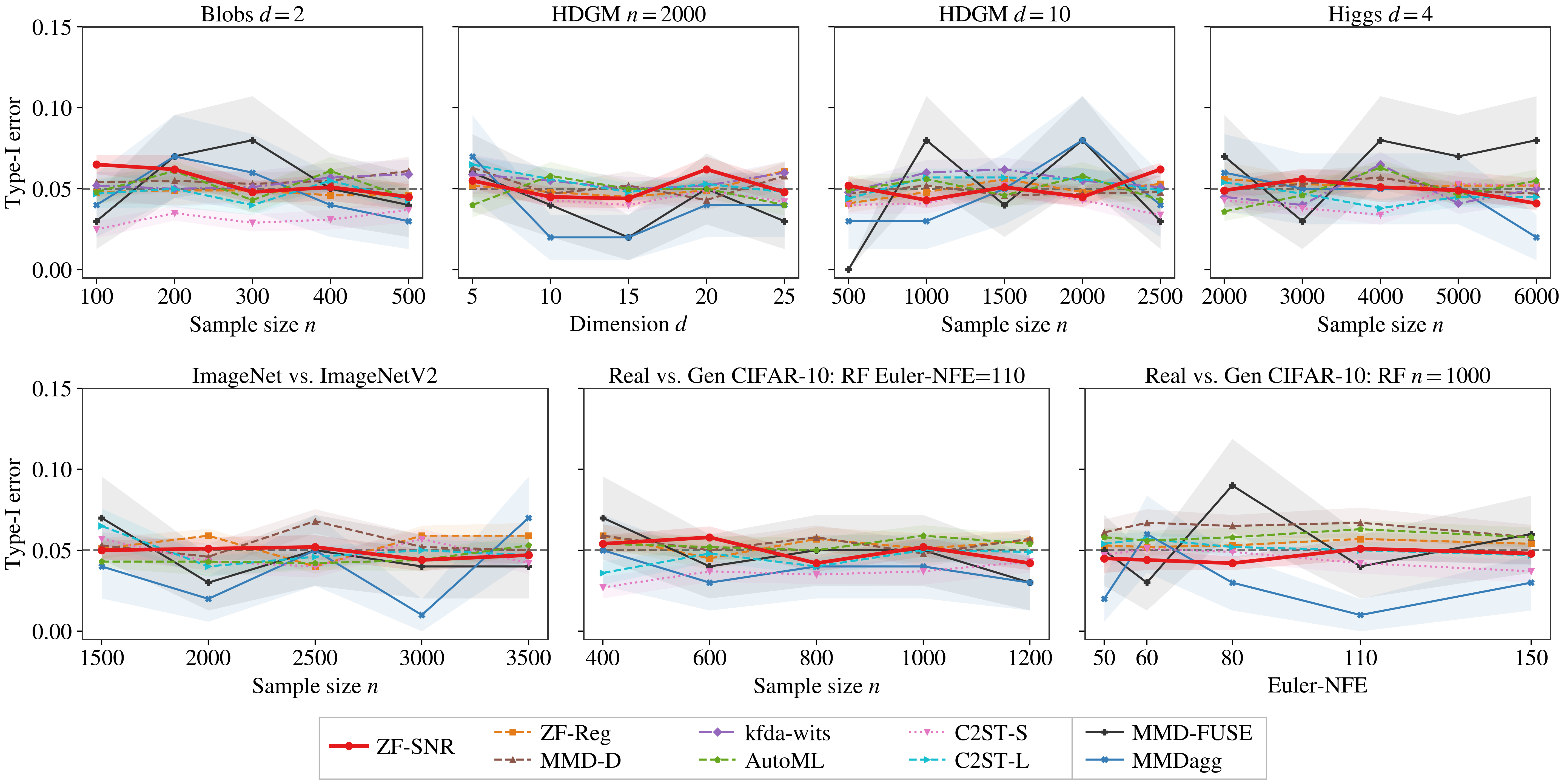}
    \caption{Type-I error results at significance level $\alpha=0.05$ for the corresponding experimental settings. See \Cref{fig:power_experiments} for the associated power curves.}
    \label{fig:type1}
\end{figure}

\begin{table}[t]
    \centering

    \begin{subtable}[t]{0.48\linewidth}
        \centering
        \caption{CIFAR-10 vs.\ CIFAR-10.1}
        \begin{tabular}{lc}
            \toprule
            Method & Type-I \\
            \midrule
            ZF-Reg   & $0.052 \pm 0.008$ \\
            ZF-SNR   & $0.052 \pm 0.006$ \\
            MMD-D    & $0.046 \pm 0.007$ \\
            C2ST-S   & $0.052 \pm 0.006$ \\
            C2ST-L   & $0.048 \pm 0.005$ \\
            MMD-FUSE & $0.060 \pm 0.024$ \\
            MMDagg   & $0.050 \pm 0.022$ \\
            AutoML  & $0.063 \pm 0.007$ \\
            \bottomrule
        \end{tabular}
        \label{tab:type1_cifar101}
    \end{subtable}
    \hfill
    \begin{subtable}[t]{0.48\linewidth}
        \centering
        \caption{Generative model experiments}
        \begin{tabular}{lc}
            \toprule
            Method & Type-I \\
            \midrule
            ZF-Reg   & $0.051 \pm 0.009$ \\
            ZF-SNR   & $0.052 \pm 0.006$ \\
            MMD-D    & $0.048 \pm 0.008$ \\
            C2ST-S   & $0.037 \pm 0.008$ \\
            C2ST-L   & $0.050 \pm 0.007$ \\
            MMD-FUSE & $0.060 \pm 0.024$ \\
            MMDagg   & $0.050 \pm 0.022$ \\
            AutoML  & $0.057 \pm 0.006$ \\
            \bottomrule
        \end{tabular}
        \label{tab:type1_genmodel}
    \end{subtable}

    \caption{
        Type-I error at significance level $\alpha=0.05$.
        Results are reported as mean $\pm$ one standard error.
        For the generative-model experiments in \Cref{tab:genmodel_power}, we report a single Type-I error result since all experiments share the same null CIFAR-10 distribution.
    }
    \label{tab:type1_image}
\end{table}

\subsection{Runtimes}
\begin{figure}[h]
    \centering
    \includegraphics[width=\linewidth]{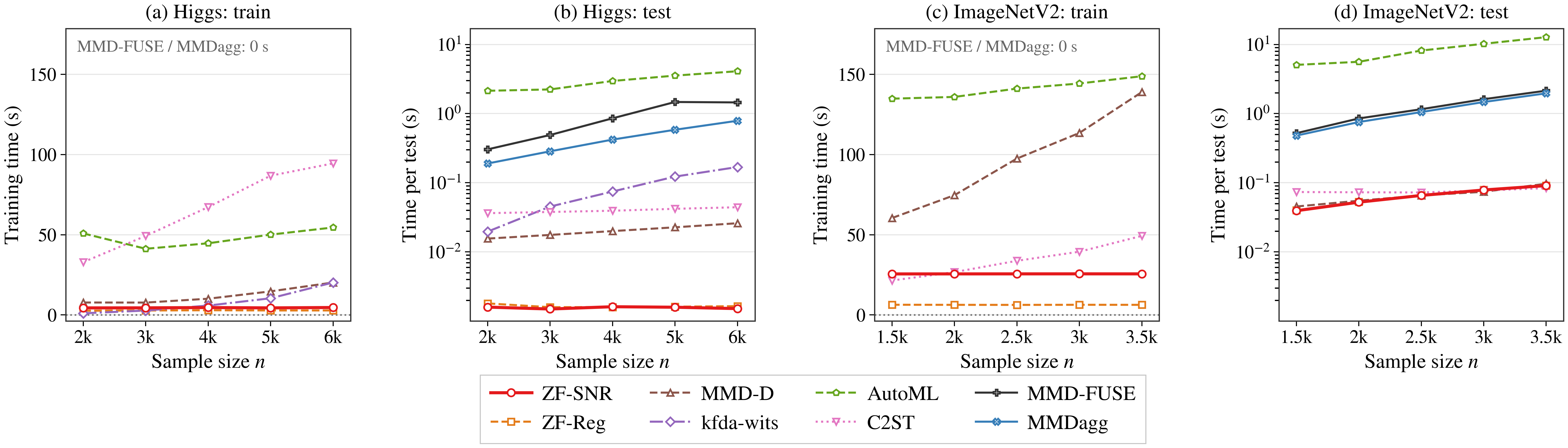}
    \caption{
    Training and testing runtimes on \texttt{Higgs} and ImageNet vs.\ ImageNetV2 as the sample size varies.
    Training time is averaged over 10 independent training runs, while testing time is averaged over 100 independent test trials.
    MMD-FUSE and MMDagg require no separate training stage and are therefore assigned zero training time. Testing times are shown on logarithmic scales.
    }
    \label{fig:runtime}
\end{figure}

We report the training and testing runtimes in \Cref{fig:runtime}. ZF-Reg and ZF-SNR evaluate the learned witness on the held-out samples in a single pass, after which the wild-bootstrap randomizations reuse the precomputed scores without additional witness evaluations, resulting in relatively low testing overhead. C2ST-S and C2ST-L share the same trained classifier, so we report their runtime jointly as C2ST.

Among the non-splitting methods, MMDagg is faster than MMD-FUSE in our experiments. In the equal-sample-size setting, the default MMDagg implementation uses wild-bootstrap calibration, whereas MMD-FUSE uses label permutations; differences in bandwidth selection and first-call overhead, including \texttt{JAX} compilation where applicable, may also affect the observed runtimes.

C2ST uses epoch-based training with small minibatches, so its number of parameter updates increases with sample size, whereas ZF-Reg and ZF-SNR use a fixed budget of $2{,}000$ updates. AutoML is configured with \texttt{time\_limit=60}, but we report the elapsed time of the complete training call, which may exceed $60$ seconds. Overall, the reported wall-clock times reflect both computational costs and the optimization and implementation choices used by each method.

\section{Ablation Study} \label{app:ablation_all}

\subsection{Variance Reduction}
\begin{figure}[t]
    \centering
    \includegraphics[width=\linewidth]{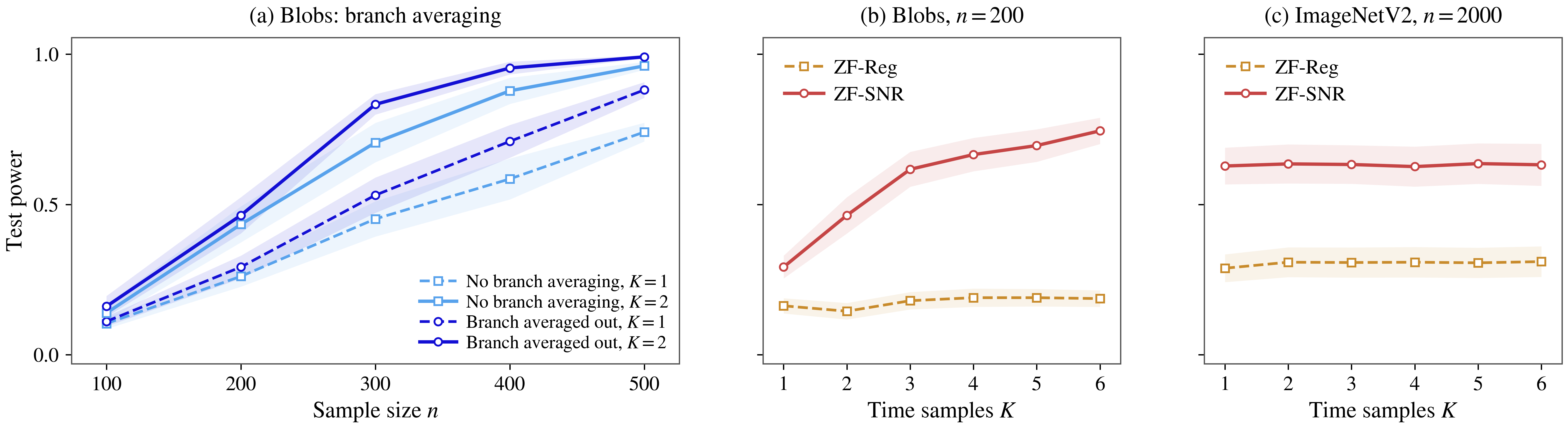}
    \caption{
    Ablation of the variance-reduction strategy.
    (a) ZF-SNR test power on \texttt{Blobs} with and without averaging out the branch variable $B$, for $K=1$ and $K=2$.
    (b) Effect of the number of time samples $K$ on \texttt{Blobs} with $n=200$.
    (c) Effect of $K$ on ImageNet vs.\ ImageNetV2 with $n=2000$.
    Lines show mean test power and shaded regions indicate $\pm1$ standard error.
    }
    \label{fig:ablation_vr}
\end{figure}

We ablate the two components of the variance-reduction strategy introduced in \Cref{sec:practical_implementation}: averaging out the branch variable $B$ and averaging over multiple time samples $K$.
As shown in \Cref{fig:ablation_vr}(a), branch averaging consistently improves test power on \texttt{Blobs} for both $K=1$ and $K=2$, while increasing $K$ provides an additional improvement.
\Cref{fig:ablation_vr}(b) further shows that increasing $K$ substantially improves ZF-SNR on \texttt{Blobs}, with diminishing gains beyond a small number of time samples, whereas ZF-Reg is less sensitive to $K$.
On ImageNet vs.\ ImageNetV2, both methods are relatively stable across the tested values of $K$ (\Cref{fig:ablation_vr}(c)).
Overall, these results support using a small number of time samples: we set $K=3$ for non-image experiments to capture most of the observed gain, and $K=2$ for image experiments where larger $K$ provides little additional benefit.

\subsection{Network Architecture} \label{app:ablation_nn}

For the non-image experiments, ZF-Reg and ZF-SNR use fully connected networks with hidden-layer widths matched to those of MMD-D and C2ST, with adaptations to the input, activation functions, and output layer for the time-conditioned witness (see Appendix~\ref{app:configurations}).
To further examine whether the strong image performance of ZF2ST can be attributed to its network architecture, we replace the original CNN backbones of MMD-D and C2ST with an adapted version of the ZF2ST CNN in \Cref{fig:cnn_arch}, and evaluate all three baseline tests on ImageNet vs.\ ImageNetV2 with $n=2000$.

The adapted backbone retains the 16-channel input convolution, one residual block, and the GroupNorm and SiLU layers. We remove the time-conditioning channel so that the network takes only RGB images as input, and replace the vector-field output layer with adaptive average pooling to a $4\times4$ spatial grid. The resulting 256-dimensional representation is passed to each baseline's original fully connected head, with the input dimension adjusted accordingly.
For MMD-D, the head outputs 300-dimensional features; for C2ST, it consists of a 300-unit hidden layer with ReLU activation followed by a two-class softmax output so that both match the output dimension \citet[Fig. 8 \& 9]{liu2020learning}.
For each baseline, the training objective, optimization settings, and testing procedure are held fixed across the two backbone variants.
The total numbers of parameters, including the output heads, are $712{,}588/82{,}284$ for MMD-D and $713{,}190/82{,}886$ for C2ST using the official/adapted backbones, respectively. ZF-Reg and ZF-SNR each use $5{,}763$ parameters.

As shown in \Cref{tab:ablation_nn}, the adapted ZF2ST backbone does not improve MMD-D, while both C2ST-S and C2ST-L achieve moderate gains in test power. 
Together with the substantially smaller parameter count of ZF2ST, these results suggest that its performance advantage is \textit{not} due to using a larger network.

\begin{table}[t]
\centering
\caption{
Backbone ablation on ImageNet vs.\ ImageNetV2 ($n=2000$, $\alpha=0.05$).
Test power (\%) is reported as mean $\pm$ SE over 10 training runs, each with 100 test trials.
MMD-D and C2ST use either their original CNN backbones \citep[Figures~8 and~9]{liu2020learning} or the adapted ZF2ST backbone described above.
ZF-Reg and ZF-SNR use their native CNN shown in \Cref{fig:cnn_arch}.
}
\begin{tabular}{lcc}
\toprule
Method & Official backbone & ZF2ST backbone \\
\midrule
MMD-D  & $11.4 \pm 1.0$ & $7.0 \pm 0.6$ \\
C2ST-S & $6.2 \pm 0.8$ & $10.0 \pm 2.3$ \\
C2ST-L & $6.2 \pm 1.0$ & $12.5 \pm 2.5$ \\
\midrule
ZF-Reg & -- & $30.8 \pm 5.0$ \\
ZF-SNR & -- & $\mathbf{63.5 \pm 6.5}$ \\
\bottomrule
\end{tabular}
\label{tab:ablation_nn}
\end{table}

\clearpage
\section{Datasets Visualization} \label{app:data_visualization}

\begin{figure}[h]
    \centering
    \includegraphics[width=0.95\linewidth]{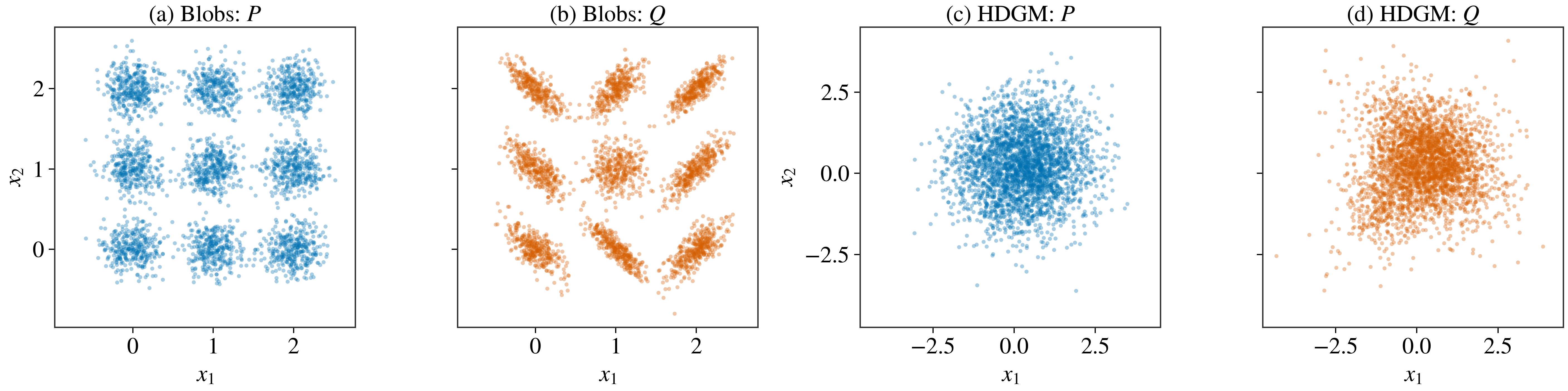}
    \caption{Illustration of the synthetic \texttt{Blobs} and \texttt{HDGM} distributions. For \texttt{Blobs}, $p$ and $q$ differ through local covariance structures across a $3\times3$ grid of Gaussian components. For \texttt{HDGM}, the distributional difference is confined to the first two dimensions.}
    \label{fig:synthetic_vis}
\end{figure}

\begin{figure}[h]
    \centering
    \includegraphics[width=0.95\linewidth]{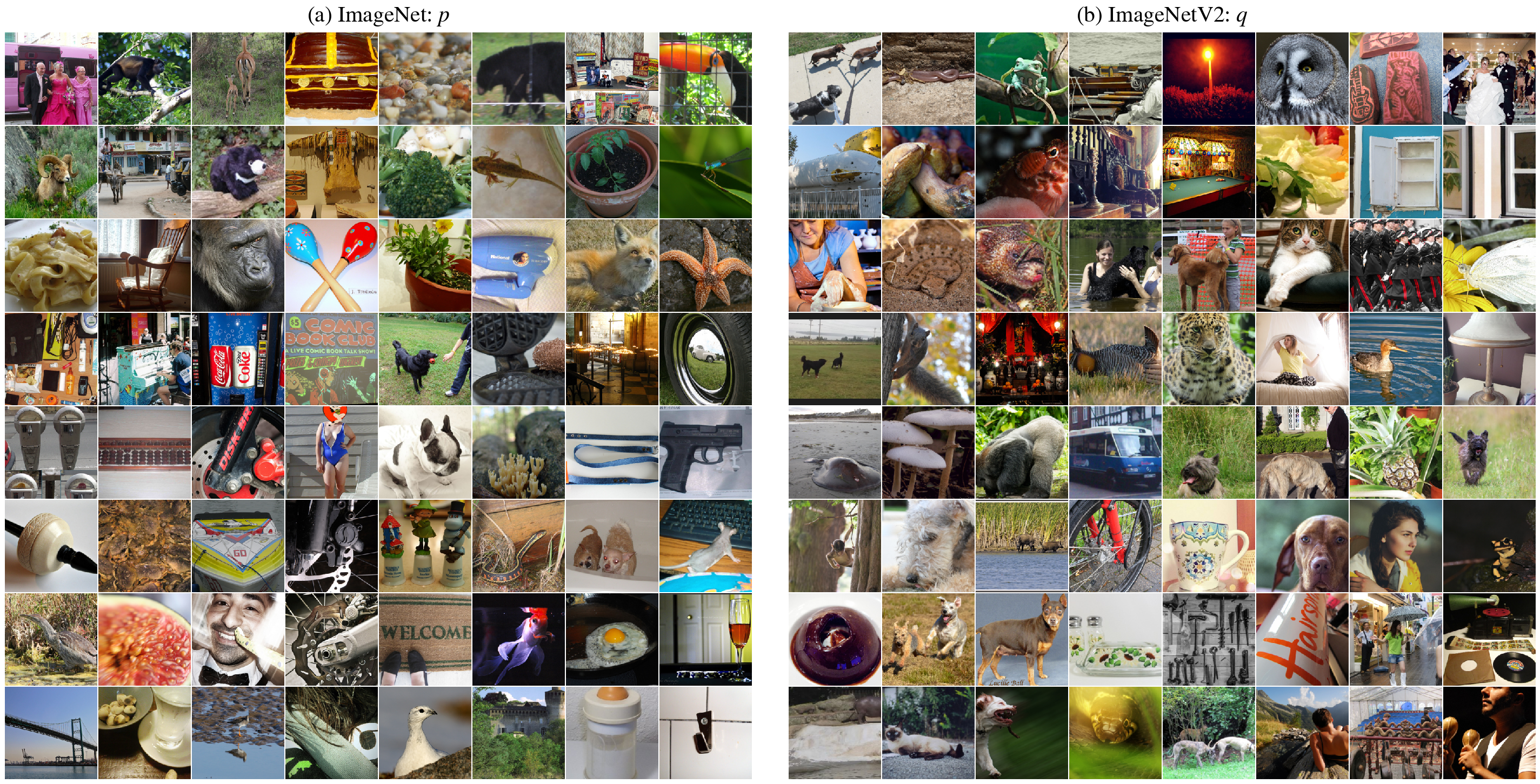}
    \caption{\textbf{Left:} ImageNet; \textbf{Right:} ImageNetV2}
    \label{fig:imgnet_vis}
\end{figure}

\begin{figure}[h]
    \centering
    \includegraphics[width=0.95\linewidth]{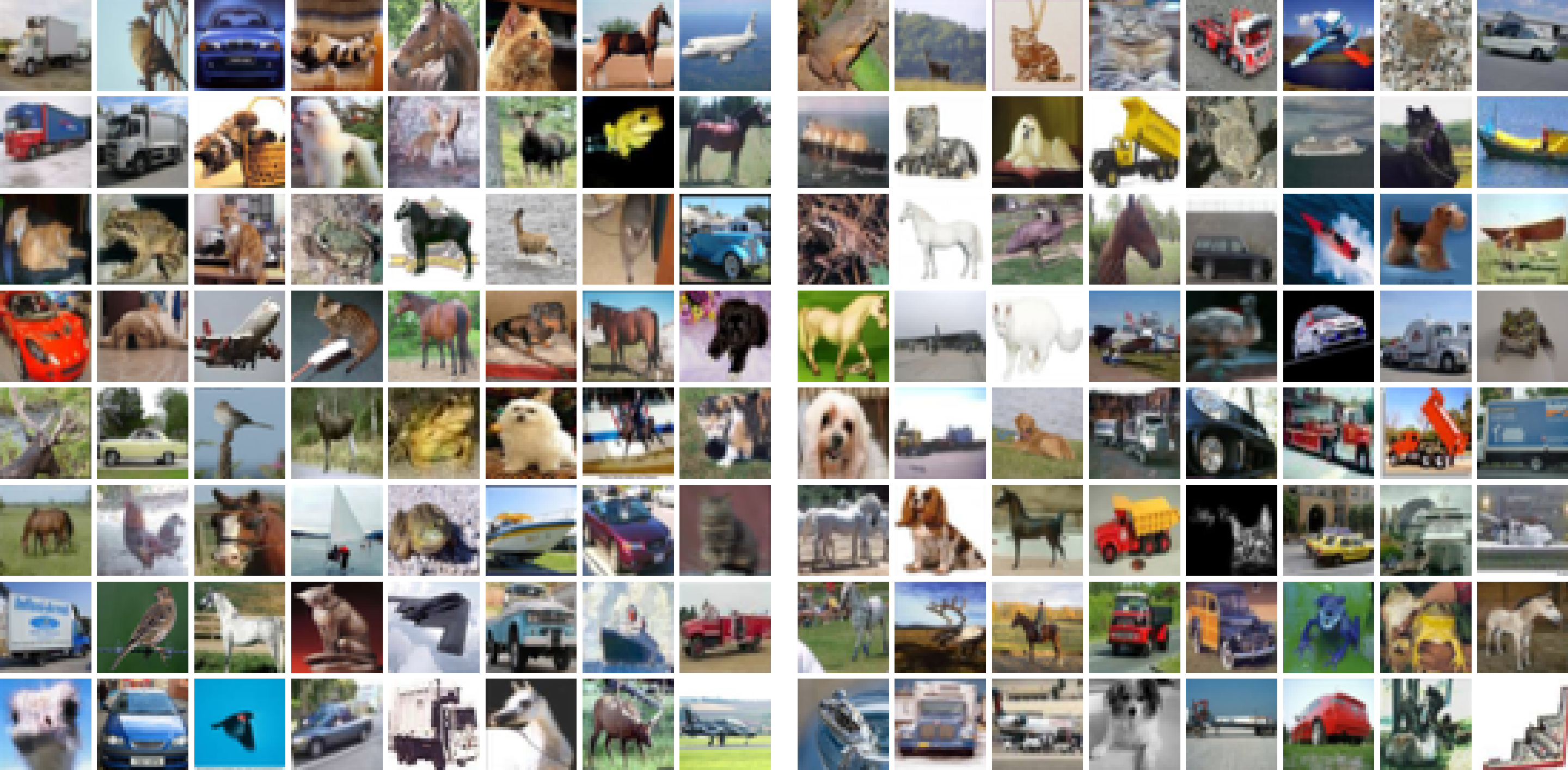}
    \caption{\textbf{Left:} CIFAR-10; \textbf{Right:} CIFAR-10.1}
    \label{fig:placeholder}
\end{figure}

\begin{figure}[h]
    \centering
    \includegraphics[width=\linewidth]{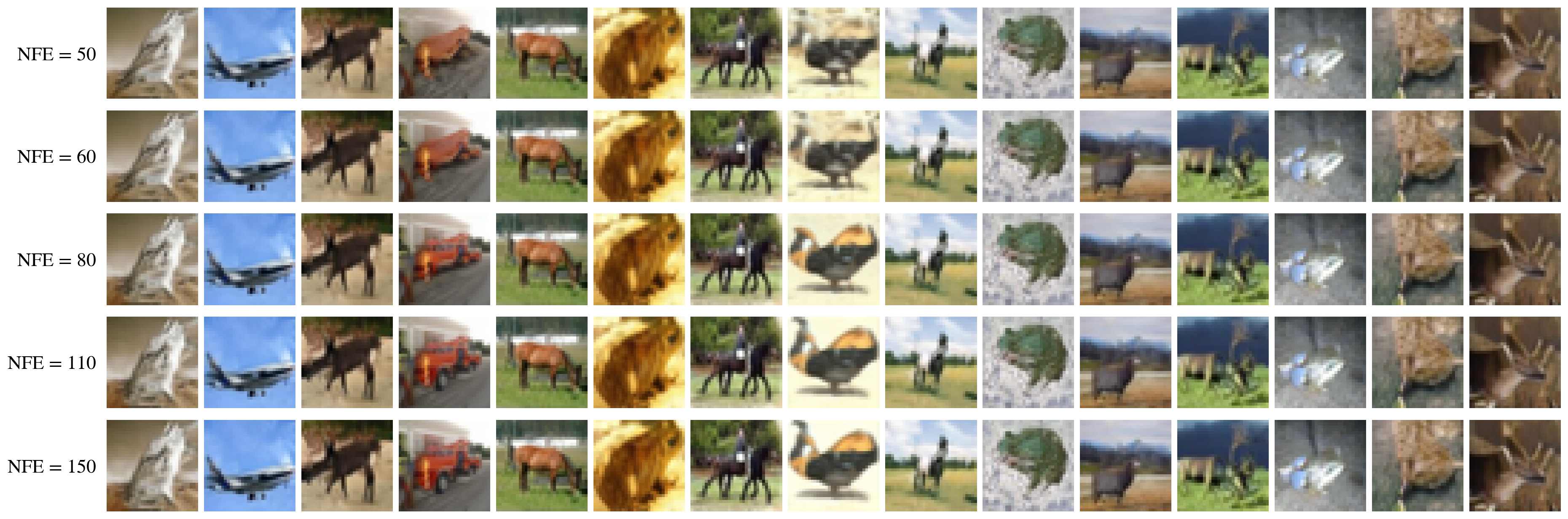}
    \caption{Samples generated by 1-Rectified Flow using the Euler solver with different NFE. Each column uses the same initial noise across NFE values, illustrating how the generated samples evolve as the number of Euler steps increases.}
    \label{fig:placeholder}
\end{figure}

\begin{figure}[h]
    \centering
    \includegraphics[width=0.9\linewidth]{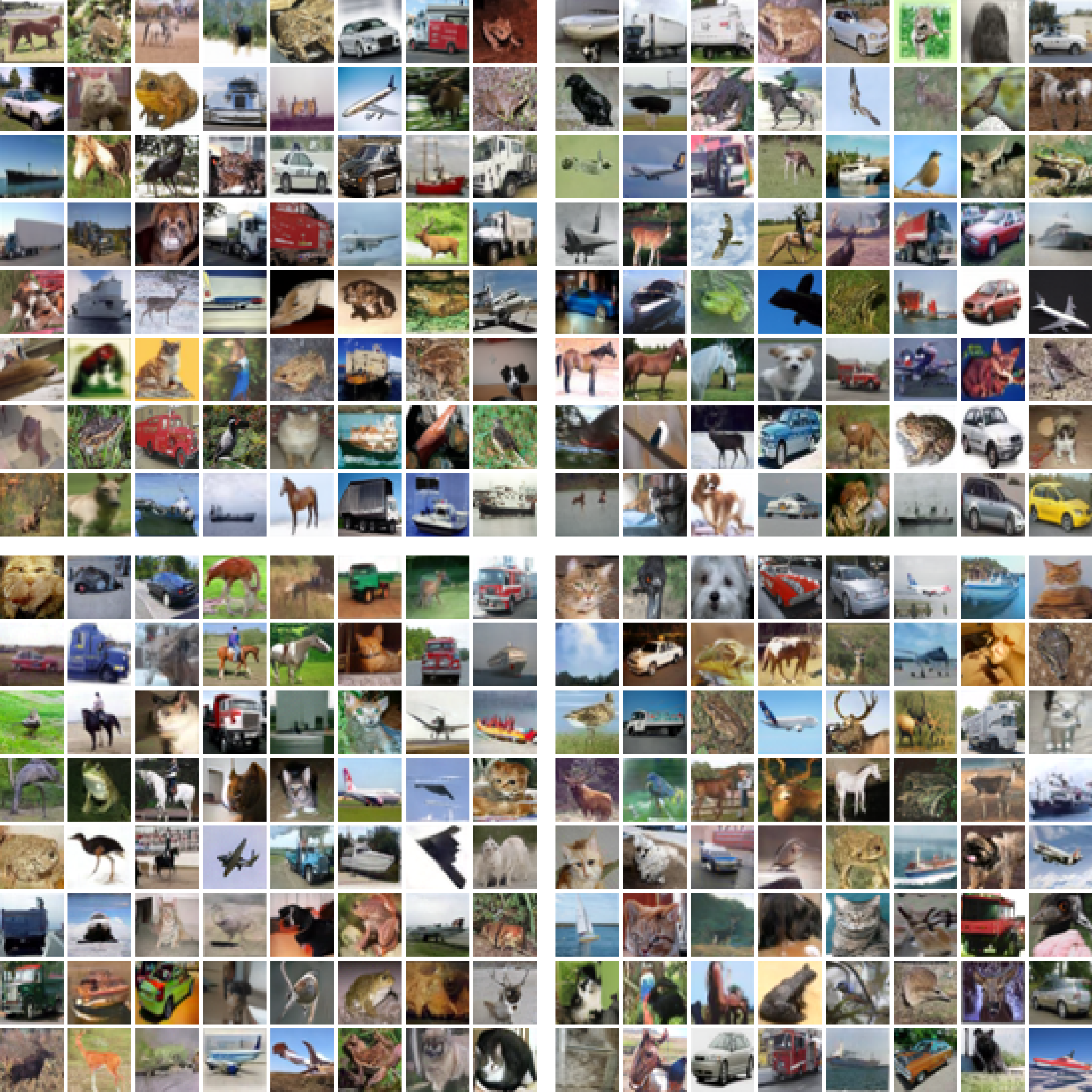}
    \caption{CIFAR-10 samples generated by (top left) ScoreSDE, (top right) RF with RK45, (bottom left) MeanFlow, and (bottom right) SGXL.}
    \label{fig:placeholder}
\end{figure}

\end{document}